\documentclass[12pt]{article}
\usepackage{graphicx}
\usepackage{amsfonts} % $\mathbb $ and $\frak $
\usepackage{amssymb}  % $\square $ and $\blacksquare$
\usepackage{mathrsfs} % \mathscr{A}
\usepackage{amsmath}  % \xrightarrow[below]{above}
\usepackage{verbatim} % for \begin{comment}
\usepackage{tcolorbox}
\usepackage{indentfirst}
\usepackage{amsmath}
\usepackage{cases}
\usepackage{diagbox}
\usepackage{multirow}
\usepackage{collcell}
\usepackage{makecell}
\usepackage{subfigure}
\def\bfx{{\bm{x}}}
\def\bfy{{\bm{y}}}
\def\bfz{{\bm{z}}}

\usepackage{color}
\usepackage{amsmath}
\usepackage{enumerate}
\usepackage{bm}
\usepackage{amsfonts}
\usepackage{amssymb}

\usepackage[colorlinks, linkcolor=blue, anchorcolor=blue, citecolor=blue]{hyperref}

\def\br{\mathbb{R}}

\def\bfx{{\bm{x}}}
\def\bfy{{\bm{y}}}
\def\bfz{{\bm{z}}}

\newcommand{\prf}{\par{\bf Proof. }}

\usepackage{setspace} 
\newtheorem{thm}{Theorem}[section]
\newtheorem{lem}[thm]{Lemma}
\newtheorem{coro}[thm]{Corollary}

\newtheorem{de}{Definition}[section]
\newtheorem{remark}{Remark}[section]
\usepackage[title]{appendix}

\usepackage{authblk}
\title{On the Expressive Power of Transformers for Maxout Networks and Continuous Piecewise Linear Functions}
\author[1]{Linyan Gu}
\author[1]{Lihua Yang}
\author[2]{Feng Zhou\thanks{Corresponding author: fengzhou@gdufe.edu.cn}}

\affil[1]{School of Mathematics, Sun Yat-sen University}
\affil[2]{School of Information Science, Guangdong University of Finance and Economics}

\begin{document}

\maketitle
\begin{abstract}
Transformer networks have achieved remarkable empirical success across a wide
range of applications, yet their theoretical expressive power remains
insufficiently understood.
In this paper, we study the expressive capabilities of
Transformer architectures.
We first establish an explicit approximation of maxout networks by Transformer networks while preserving comparable model complexity.
As a consequence, Transformers inherit the universal approximation capability of
ReLU networks under similar complexity constraints.
Building on this connection, we develop a framework to analyze the approximation of continuous piecewise linear functions by Transformers and quantitatively characterize their expressivity via the number of linear regions, which grows exponentially with depth.
Our analysis establishes a theoretical bridge between approximation
theory for standard feedforward neural networks and Transformer architectures.
It also yields structural insights into Transformers: self-attention layers
implement max-type operations, while feedforward layers realize token-wise
affine transformations.
\end{abstract}

\section{Introduction}\label{sec: intro}
Transformer networks have become a cornerstone of modern sequence modeling since their introduction by Vaswani et al.~\cite{vaswani2017attention}. They have driven major advances across a wide range of natural language processing (NLP) tasks, including machine translation~\cite{vaswani2017attention,brown2020language}, language modeling~\cite{radford2018improving,radford2019language}, and question answering~\cite{devlin2019bert,brown2020language}. 
Beyond NLP, Transformers have also been successfully applied to a wide range of domains, such as computer vision~\cite{dosovitskiy2020image}, signal and speech processing~\cite{kim2022squeezeformer}, and healthcare applications~\cite{nerella2024transformers}.
Despite these empirical successes, the theoretical understanding of Transformer architectures remains limited, and many fundamental questions about their expressive power are still open.

A Transformer block consists of two main components: a self-attention layer and a token-wise feedforward layer. 
Both layers reuse the same parameters across tokens, and interactions between tokens are captured exclusively through pairwise dot products in the self-attention mechanism.
This parameter sharing makes it feasible to train large-scale models efficiently, in contrast to the sequential computation required by RNNs and LSTMs.
At the same time, parameter sharing and the restricted form of token interactions pose significant challenges for theoretical analysis~\cite{Yun2020Are,kajitsuka2023transformers,jiao2025transformers}.
To treat this challenge, \cite{Yun2020Are} introduced the concept of
``contextual mapping'', which aggregates dependencies among all tokens into
token-based quantities through the self-attention mechanism.

Our study is motivated by the intrinsic connection between the self-attention
mechanism and the max operation
$
(x_1,\dots,x_k)\ \mapsto\ \max\{x_1,\dots,x_k\},
$
which inspires us to investigate the approximation of maxout networks by
Transformer architectures.
A maxout network is a feedforward neural network equipped with max activation
functions, originally proposed in~\cite{goodfellow2013maxout} as a generalization
of networks with rectified linear units (ReLUs), and is known to exactly represent
continuous piecewise linear (CPWL) functions \cite{montufar2014number,MontufarSiam2022}.

By establishing a systematic approximation of maxout networks using Transformer
models, we obtain a principled framework for analyzing the ability of
Transformers to approximate CPWL functions.
Moreover, since the number of linear regions is a fundamental measure of
expressivity for CPWL models, this framework enables a quantitative characterization of
Transformer expressivity in terms of linear region complexity.
From this perspective, our results offer a new viewpoint on the expressive power
of Transformer architectures, and align with recent advances in approximation
theory that characterize neural network expressivity through the growth of the
number of linear regions
\cite{telgarsky2016benefits,arora2016understanding,hinz2019framework,serra2018bounding,xiong2020number,
MontufarSiam2022,chen2022improved,xiong2024number,montufar2014number,pascanu2013number}.

In addition, our proof techniques clarify the distinct roles played by the
self-attention and feedforward layers in Transformer architectures.
Specifically, the self-attention layer is responsible for implementing
max-type operations, while the feedforward layer realizes token-wise affine
transformations.
To mitigate the limitations induced by parameter sharing across tokens in the
feedforward layers, rather than relying on the notion of contextual mapping,
we introduce a token-wise shift that is applied repeatedly along the depth of
the Transformer.
This mechanism enhances the design flexibility and expressive capacity of
token-wise feedforward networks.

Our contributions are summarized as follows.
\begin{itemize}
\item We provide an explicit construction of Transformer networks that approximate both shallow and deep maxout networks in the $L^\infty$ norm while preserving comparable model complexity. 
Since maxout networks strictly generalize ReLU networks, this result further
implies that Transformers possess the universal approximation capability for
ReLU networks under similar complexity constraints.
These findings establish a theoretical bridge between the approximation theory
of standard feedforward neural networks and Transformer architectures.

\item Building on this connection, we develop a principled framework for
analyzing the approximation of continuous piecewise linear (CPWL) functions by
Transformers.
Within this framework, we obtain a quantitative characterization of Transformer
expressivity in terms of the number of linear regions, and show that this number
grows exponentially with network depth.

\item We provide structural insights into Transformer architectures.
To mitigate the limitations induced by parameter sharing in the feedforward
layers, we introduce a token-wise shift applied repeatedly along the depth of
the Transformer, rather than relying on the notion of contextual mapping.
This mechanism enhances the design flexibility and expressive capacity of
token-wise feedforward networks.
Moreover, we show that the two core components of Transformers play distinct
roles: self-attention layers implement max-type operations, while feedforward
layers realize token-wise affine transformations.
\end{itemize}

\subsection{Related Works}

\paragraph{Expressive capacity of feedforward neural networks.}
The study of expressive capacity of feedforward neural networks (FNNs) dates back
to the classical universal approximation theorems established around the 1990s~\cite{cybenko1989approximation,hornik1991approximation,funahashi1989approximate}.
In recent years, there has been substantial progress in the theoretical analysis
of approximation properties of FNNs.
A central theme of this line of work is to understand how architectural features, such as network width and
depth~\cite{yarotsky2017error,shen2020deep,shen2022optimal,lu2021deep}, as well as the choice of activation functions~\cite{siegel2020approximation,yang2025optimal,yarotsky2021elementary}, affect expressive power.
A wide range of function spaces has been considered, including H\"older spaces~\cite{shen2020deep,shen2022optimal},
$C^s$ spaces~\cite{lu2021deep}, Sobolev spaces~\cite{yarotsky2017error}, Barron spaces~\cite{klusowski2018approximation}, among others, revealing
the broad approximation capabilities of FNNs.

For FNNs equipped with piecewise linear activation functions, such as ReLU and
maxout, an influential line of research characterizes expressivity through the
number of linear regions of the represented functions.
This perspective was initiated in~\cite{montufar2014number,pascanu2013number}, showing that deep
networks can represent functions with exponentially more linear regions than
shallow networks with a comparable number of units or parameters.
Subsequent works have further refined these results and extended them to various
architectures; see, for example,
\cite{telgarsky2016benefits,arora2016understanding,hinz2019framework,serra2018bounding,xiong2020number,
MontufarSiam2022,chen2022improved,xiong2024number}.
In this work, we establish a systematic approximation of maxout networks using
Transformer architectures, providing a bridge between the approximation theory
of maxout (and ReLU) networks and Transformer models.

\paragraph{Expressive capacity of Transformers.}
Compared with standard feedforward neural networks (FNNs), the theoretical
analysis of Transformer architectures is considerably more challenging due to
parameter sharing and the restricted form of token interactions.
Transformers must account for global contextual dependencies through
self-attention.
Yun et al.~\cite{Yun2020Are} proposed to aggregate these dependencies into a
token-wise quantity, referred to as a ``contextual mapping'', which is produced by self-attention and subsequently mapped to the desired
output via a feedforward layer.
This approach led to the first universal approximation theorem for Transformer
models.
Their results were subsequently extended to sparse-attention Transformers in
\cite{yun2020n,zaheer2020big}.
More recently, Kajitsuka et al.~\cite{kajitsuka2023transformers} showed that even
single-layer Transformers can achieve universal approximation by sufficiently
increasing the network width.

In contrast to these works, our approach does not rely on the notion of contextual mapping.
Instead, to mitigate the limitations of parameter sharing in the feedforward layers,
we introduce a token-wise shift applied repeatedly along the Transformer depth,
which enhances the design flexibility and expressive capacity of
token-wise feedforward networks.

Beyond universal approximation results, many intrinsic theoretical properties of
Transformer architectures have been investigated.
One important line of work focuses on establishing explicit approximation rates.
Takakura and Suzuki~\cite{takakura2023approximation} established approximation
rates for smooth infinite-length sequence-to-sequence mappings.
Jiang et al.~\cite{jiang2024approximation} introduced a novel complexity measure
and derived explicit Jackson-type approximation rates for Transformers.
Wang et al.~\cite{wang2024understanding} further quantified approximation rates,
highlighting the roles of network depth and the number of attention heads.
Approximation rates for H\"older continuous function classes were obtained
in~\cite{jiao2025transformers}.
Another line of research studies the theoretical properties of the self-attention
mechanism itself, including analyses of the rank structure of attention matrices
\cite{dong2021attention,bhojanapalli2020low}, and the Lipschitz continuity of self-attention
\cite{kim2021lipschitz,castin2023smooth}.

In addition, the expressive power of Transformers has been compared with that of
other neural architectures.
For instance, it has been shown that attention layers can emulate convolution
operations
\cite{cordonnier2019relationship,li2021can}, and that Transformers exhibit
distinct expressive behaviors compared with recurrent neural networks
\cite{wang2024understanding,jiang2024approximation}.

In this paper, we investigate the approximation of maxout networks by Transformer
architectures, motivated by the intrinsic connection between self-attention and
max operations, and further obtain the approximation of ReLU networks.
This perspective provides a principled framework for analyzing the ability of
Transformers to approximate continuous piecewise linear (CPWL) functions and
enables a quantitative characterization of Transformer expressivity in terms of
the number of linear regions.

\subsection{Organization of This Paper}
The remainder of this paper is organized as follows.
In Section~\ref{sec:pre}, we introduce the notation used throughout the paper and
review the Transformer architecture.
Section~\ref{sec: appro_maxout} presents the universal approximation of maxout networks by
Transformer architectures.
Based on this result, Section~\ref{sec:CPWL} further investigates the ability of
Transformers to approximate continuous piecewise linear functions.
Finally, Section~\ref{sec:conclusion} concludes the paper.
Detailed proofs are deferred to the appendix.

\section{Preliminaries}\label{sec:pre}
\subsection{Notations}\label{sec:nota}

We summarize the notations used throughout this paper as follows.
\begin{itemize}
\item Let $\mathbb{R}$ and $\mathbb{N}$ denote the sets of real numbers and natural
numbers, respectively.

\item Let $X=(\bm{x}_1,\ldots,\bm{x}_T)\in \mathbb{R}^{d\times T}$ denote a sequence of
length $T$, where each embedding $\bm{x}_t\in \mathbb{R}^d$.

\item For $n\in \mathbb{N}$, define $[n]:=\{1,\ldots,n\}$.

\item Let $\bm{1}_n$ (resp.\ $\bm{0}_n$) denote the length-$n$ all-ones (resp.\ all-zeros)
vector.

\item Let $\bm{1}_{n\times m}$ (resp.\ $\bm{0}_{n\times m}$) denote the $n\times m$
all-ones (resp.\ zero) matrix, and let $I_n$ denote the $n\times n$ identity
matrix.

\item When the dimensions are clear from the context, $\bm{1}$ and $\bm{0}$ denote
all-ones and zero vectors or matrices, respectively.

\item For a vector $\bm{x}\in\mathbb{R}^n$, let $(\bm{x})_i$ denote its $i$-th entry,
and let $(\bm{x})_{i:j}$ denote the subvector consisting of entries $i$ through
$j$, where $i,j\in[n]$ and $i<j$.

\item For a matrix $A\in\mathbb{R}^{n\times m}$, let $(A)_{i,j}$ denote its $(i,j)$-th
entry, and let $A_{i,:}$ and $A_{:,j}$ denote its $i$-th row and $j$-th column,
respectively, where $i\in[n]$ and $j\in[m]$.
Moreover, $(A)_{i:j,:}$ denotes the submatrix consisting of rows $i$ through $j$,
and $(A)_{:,i:j}$ denotes the submatrix consisting of columns $i$ through $j$,
for $i<j$.

\item For a vector $\bm{x}=(x_1,\ldots,x_n)^\top\in \mathbb{R}^n$ and $p\in[1,\infty)$,
the $\ell^p$ norm is defined by
\[
\|\bm{x}\|_p:=\Bigl(\sum_{i=1}^n |x_i|^p\Bigr)^{1/p},
\]
and the $\ell^\infty$ norm is defined by
\[
\|\bm{x}\|_\infty:=\max_{1\le i\le n}|x_i|.
\]

\item For $p\in[1,\infty)$, the induced $\ell^p$ matrix norm of
$A=(a_{ij})_{i\in[n],\,j\in[m]}$ is defined by
\[
\|A\|_p:=\sup_{\|\bfx\|_p=1}\|A\bfx\|_p.
\]
In particular,
\[
\|A\|_1:=\max_{j\in[m]}\sum_{i=1}^n |a_{ij}|, \qquad
\|A\|_\infty:=\max_{i\in[n]}\sum_{j=1}^m |a_{ij}|.
\]

\item We use $C(a,b,\dots)$ to denote a positive constant depending only on the
parameters $a,b,\dots$.

\item For $X=(\bm{x}_1,\ldots,\bm{x}_T)\in \mathbb{R}^{d\times T}$, define the
vectorization operator $\mathrm{Vec}:\mathbb{R}^{d\times T}\to \mathbb{R}^{dT}$ by
\[
\mathrm{Vec}(X):=(\bm{x}_1^\top,\ldots,\bm{x}_T^\top)^\top.
\]
Its inverse mapping $\mathrm{Vec}^{-1}_{d,T}:\mathbb{R}^{dT}\to\mathbb{R}^{d\times T}$
is defined by
\[
[\mathrm{Vec}^{-1}_{d,T}(\bm{x})]_{:,t}
=(x_{(t-1)d+1},\ldots,x_{td})^\top, \quad t\in[T],
\]
for $\bm{x}=(x_1,\ldots,x_{dT})^\top$.

\item For $\bm{x}=(x_1,\ldots,x_n)^\top\in \mathbb{R}^n$, define
$\max(\bm{x}):=\max_{1\le i\le n} x_i$.

\item The rectified linear unit (ReLU) is defined by
$\mathrm{ReLU}(x):=\max\{0,x\}$.

\item We consider two activation functions in the self-attention module: the
hardmax and the scaled softmax with scaling parameter $\lambda>0$, defined by
\[
[\sigma_H(\bm{x})]_i :=
\begin{cases}
|\arg\max_{j\in[T]} x_j|^{-1}, & i\in \arg\max_{j\in[T]} x_j,\\
0, & \text{otherwise},
\end{cases}
\]
and
\[
[\sigma_S^\lambda(\bm{x})]_i
:=\frac{e^{\lambda x_i}}{\sum_{j=1}^T e^{\lambda x_j}},
\quad i\in[T],
\]
for $\bm{x}=(x_1,\ldots,x_T)^\top\in\mathbb{R}^T$.

\item Let $\mathcal{N}_H$ and $\mathcal{N}_S^\lambda$ denote Transformer networks
equipped with the hardmax activation $\sigma_H$ and the scaled softmax activation
$\sigma_S^\lambda$, respectively.
When appearing as a pair, the two networks are identical except for the choice of
activation function.

\item For a sequence-to-sequence function $f:\Omega\to\mathbb{R}^{m\times T}$ with
$\Omega\subset\mathbb{R}^{n\times T}$, let $(f)_t:\Omega\to\mathbb{R}^m$ denote its
$t$-th coordinate function, defined by
\[
(f)_t(X):=[f(X)]_{:,t}, \quad t\in[T].
\]

\item For a sequence-to-sequence or vector-valued function $f$ defined on
$\Omega\subset\mathbb{R}^{n\times T}$ or $\Omega\subset\mathbb{R}^n$, the
$L^\infty$ norm is defined by
\[
\|f\|_{L^\infty(\Omega)}
:=\operatorname*{ess\,sup}_{X\in\Omega}\|f(X)\|_\infty.
\]

\item A function $f$ is called convex if $\Omega$ is convex and
\[
f(\theta x+(1-\theta)y)
\le \theta f(x)+(1-\theta)f(y),
\quad \forall\, x,y\in\Omega,\ \theta\in[0,1].
\]

\item If $f$ is Lipschitz continuous on $\Omega$, its Lipschitz constant is defined
by
\[
\mathrm{Lip}(f;\Omega)
:=\sup_{\substack{x,y\in\Omega\\x\neq y}}
\frac{\|f(x)-f(y)\|_\infty}{\|x-y\|_\infty}.
\]
\end{itemize}

\subsection{Transformer Networks}

Transformers are widely used for sequence modeling, aiming to learn mappings
between input and output sequences.
A Transformer network comprises three main components: positional embedding,
stacked Transformer blocks, and a linear readout layer.

\paragraph{Transformer Block.}
Each Transformer block consists of a self-attention layer followed by a
token-wise feedforward layer.
We consider an input space $\mathcal{X} \subset \mathbb{R}^{d \times T}$ consisting
of sequences of $T$ tokens, each represented as a $d$-dimensional embedding.
Given input $X \in \mathcal{X}$, the output of a single block is computed as
\begin{equation}\label{eq:transformer_block}
\begin{aligned}
\text{Attn}(X) &= X + \sum_{h=1}^H W_O^h W_V^h X \, \sigma\left[(W_K^h X)^\top W_Q^h X\right], \\
\text{FF}(\text{Attn}(X)) &= \text{Attn}(X)  + W_2 \, \mathrm{ReLU}(W_1 \text{Attn}(X) + \bm{b}_1 \bm{1}_T^\top) + \bm{b}_2 \bm{1}_T^\top,
\end{aligned}
\end{equation}
where, for the $h$-th attention head, $W_K^h$, $W_Q^h$, $W_V^h \in \mathbb{R}^{k \times d}$ denote the key, query, and value projection matrices, respectively, and $W_O^h \in \mathbb{R}^{d \times k}$ is the output projection matrix. The feedforward layer parameters are $W_1 \in \mathbb{R}^{r \times d}$, $W_2 \in \mathbb{R}^{d \times r}$, with bias vectors $\bm{b}_1 \in \mathbb{R}^r$ and $\bm{b}_2 \in \mathbb{R}^d$.
The activation function $\sigma$ in the self-attention layer is applied column-wise,
yielding a column-stochastic attention matrix.
We consider two types of activation functions in the self-attention module: hardmax
$\sigma_H$ and scaled softmax $\sigma_S^\lambda$ with scaling parameter $\lambda>0$.
The ReLU activation in the feedforward layer is applied element-wise to each entry.

Here, $d$ is the embedding dimension, $H$ is the number of heads, $k$ is the head size,
and $r$ is the hidden layer size of the feedforward layer.
We denote the complexity of a Transformer block by $\bm{s} = (d, k, H, r)$, and let
$\mathcal{F}_{\text{tf}}^{\bm{s}} \subset \{ f: \mathcal{X} \to \mathbb{R}^{d \times T} \}$
denote the class of functions computed by a Transformer block of complexity $\bm{s}$.
Note that the parameters in both attention and feedforward modules are shared
across tokens, rendering the layer equivariant to permutations.
Token interactions are mediated solely through pairwise dot products in the attention mechanism.

\paragraph{Transformer Network.}
Let $E : \mathbb{R}^{n \times T} \to \mathbb{R}^{d \times T}$ denote the positional embedding, and
$C : \mathbb{R}^{d \times T} \to \mathbb{R}^{m \times T}$ the linear readout.
A Transformer network with $L$ blocks of complexity $\bm{s}$ is defined as
\[
\mathcal{TFN}_{n,m}(L, d, k, H, r) := \Bigl\{ \mathcal{T}: \mathbb{R}^{n \times T} \to \mathbb{R}^{m \times T} \;\Big|\;
\mathcal{T} = C \circ f^L \circ \cdots \circ f^1 \circ E,\, f^i \in \mathcal{F}_{\text{tf}}^{\bm{s}} \Bigr\}.
\]
When the self-attention activation is fixed to hardmax, the corresponding class is
denoted by $\mathcal{TFN}^\text{hard}_{n,m}(L, d, k, H, r)$.
Given a hardmax-based Transformer $\mathcal{N}_H$, the associated scaled softmax version
with parameter $\lambda>0$ is denoted by $\mathcal{N}_S^\lambda$, obtained by
replacing $\sigma_H$ with $\sigma_S^\lambda$.

\paragraph{Positional Embedding.}
Transformer layers are inherently permutation-equivariant~\cite{Yun2020Are}.
To break this symmetry, positional embeddings are commonly employed~\cite{vaswani2017attention,Yun2020Are}.
In this paper, we adopt a simple positional encoding: the input is linearly projected
and then shifted by token-specific offsets~\cite{jiang2024approximation},
formally modeled as $E(X) = AX + B$ with $A \in \mathbb{R}^{n \times d}$ and $B \in \mathbb{R}^{d \times T}$.
This ensures that each token is mapped to a distinct domain, eliminating permutation
equivariance and enhancing the design flexibility of token-wise feedforward networks.

\paragraph{Auxiliary Token.}
We employ an auxiliary token, a technique widely used in Transformer applications.
For example, the classification token ([CLS]) in BERT~\cite{devlin2019bert} serves this purpose.
Concretely, for a context of size $T$, the input is augmented to $T+1$ tokens by
concatenating a constant vector as the auxiliary token.
The output corresponding to this token is disregarded in subsequent processing.
For simplicity, we omit further elaboration, as its functional role is made
explicit in the proofs.

\section{Approximation for Maxout Networks by Transformers}\label{sec: appro_maxout}

A typical fully connected feedforward neural network (FNN) is a composition of
stacked layers, where each layer applies an affine transformation followed by a
nonlinear activation. Formally, an FNN with input dimension $n$, output dimension $m$, 
and $L$ hidden layers of widths $d_1, \dots, d_L$ computes a map
\begin{equation}\label{eq: FFNN}
T: \mathbb{R}^n \to \mathbb{R}^m := g \circ T^{(L)} \circ \cdots \circ T^{(1)}, 
\quad \text{with} \quad T^{(\ell)}(\bm{x}) := \sigma(W_\ell \bm{x} + \bm{b}_\ell),
\end{equation}
where $\sigma$ denotes the activation function, $W_\ell \in \mathbb{R}^{d_\ell \times d_{\ell-1}}$,
$\bm{b}_\ell \in \mathbb{R}^{d_\ell}$ for $\ell=1,\dots,L$ with $d_0 = n$, and 
$g(\bm{x}) = W \bm{x} + \bm{b}$ with $W \in \mathbb{R}^{m \times d_L}$, $\bm{b} \in \mathbb{R}^m$.

A maxout network~\cite{goodfellow2013maxout} is a type of FNN in which each neuron outputs the maximum
over a set of affine functions. Specifically, a rank-$p$ maxout layer with $n$ inputs
and $m$ outputs defines the function class
\begin{equation} \label{eq:Tmax}
\begin{aligned}
\mathcal{T}_{\text{max}}(n, p, m) := \Big\{ T\colon \mathbb{R}^n \to \mathbb{R}^m \,\Big|\,
& T(\bm{x}) = \big[\max(W^1 \bm{x} + \bm{b}^1), \dots, \max(W^m \bm{x} + \bm{b}^m)\big]^\top, \\
& W^i \in \mathbb{R}^{p \times n},\ \bm{b}^i \in \mathbb{R}^p,\ i \in [m] \Big\}.
\end{aligned}
\end{equation}

A single maxout layer is often referred to as a shallow maxout network.
A deep maxout network is constructed by stacking multiple such layers. 
Formally, a rank-$p$ deep maxout network with input dimension $d_0$ and $L$ layers of
widths $d_1, \dots, d_L$ computes
\begin{equation} \label{eq:maxNet}
T := T^{(L)} \circ \cdots \circ T^{(1)}, 
\quad \text{with } T^{(\ell)} \in \mathcal{T}_{\text{max}}(d_{\ell-1}, p, d_\ell),\ \ell \in [L].
\end{equation}

To accommodate sequence-to-sequence mappings, we adopt a convention of vectorizing
inputs and outputs. Specifically, a rank-$p$ maxout layer mapping from
$\mathbb{R}^{n_1 \times n_2}$ to $\mathbb{R}^{m_1 \times m_2}$ is defined by
\[
\begin{aligned}
\mathcal{T}_{\text{max}}(n_1 \times n_2, p, m_1 \times m_2) := \Big\{ T\colon \mathbb{R}^{n_1 \times n_2} \to \mathbb{R}^{m_1 \times m_2} \,\Big|\,
& T(X) = \mathrm{Vec}^{-1}_{m_1, m_2} \, \tilde{T}(\mathrm{Vec}(X)), \\
& \tilde{T} \in \mathcal{T}_{\text{max}}(n_1 n_2, p, m_1 m_2) \Big\}.
\end{aligned}
\]

This vectorization convention naturally extends to deep maxout networks
defined in \eqref{eq:maxNet} and to feedforward networks defined in \eqref{eq: FFNN},
where both inputs and outputs are treated in the same vectorized form.

We construct a three-layer Transformer network that approximates a maxout layer
with arbitrary accuracy. 
As defined in Section~\ref{sec:nota}, we consider paired Transformer networks
$\mathcal{N}_H$ and $\mathcal{N}_S^\lambda$, which are identical in architecture,
except that $\mathcal{N}_H$ employs the hardmax activation while
$\mathcal{N}_S^\lambda$ uses the scaled softmax activation. 
This convention is adopted throughout all subsequent theorems.

\begin{thm}[Approximation of a Maxout Layer with $p \le T$]\label{thm_shallow_max}
Let $p \le T$, let $\Omega \subset \mathbb{R}^{n \times T}$ be compact, and let
$f \in \mathcal{T}_{\mathrm{max}}(n \times T, p, m \times T)$. 
Then, there exists a hardmax-based Transformer network
\[
\begin{aligned}
\mathcal{N}_H \in \mathcal{TFN}^{\mathrm{hard}}_{n,m}\bigl(
&L = 3, \\
&d = \max\{n,\, m(Tp+T+1)\}+T+1, \\
&k = 2, \\
&H = Tmp, \\
&r = 4(n+1)(T+3)
\bigr),
\end{aligned}
\]
such that
\[
\mathcal{N}_H\big|_{\Omega} = f\big|_{\Omega}.
\]
Moreover, for any $\epsilon > 0$, the corresponding softmax-based Transformer
$\mathcal{N}_S^\lambda$ satisfies
\[
\|\mathcal{N}_S^\lambda - f\|_{L^\infty(\Omega)} < \epsilon,
\]
provided that the scaling parameter satisfies $\lambda = \mathcal{O}(1/\epsilon)$.
\end{thm}

\paragraph{Proof Sketch.}
Let $X_{\mathrm{Vec}} = \mathrm{Vec}(X)$. 
From \eqref{eq:Tmax}, the $k$-th column of $f$ can be written as
\[
(f)_k(X)
=
\bigl[
\max(W^{1,k} X_{\mathrm{Vec}} + \bm{b}^{1,k}),
\dots,
\max(W^{m,k} X_{\mathrm{Vec}} + \bm{b}^{m,k})
\bigr]^\top,
\]
where $X = (\bm{x}_1,\dots,\bm{x}_T) \in \Omega$, $W^{i,k} \in \mathbb{R}^{p \times nT}$, and
$\bm{b}^{i,k} \in \mathbb{R}^p$ for $i \in [m]$. 
The construction proceeds in two steps:

\paragraph{Step 1: Approximating the affine maps.}
We first approximate the set of affine functions
$\{W^{i,k} X_{\mathrm{Vec}} + \bm{b}^{i,k}\}_{i=1}^m$ using a combination of
a feedforward layer and a self-attention layer.

Define
$
V^{(i,j)} := \mathrm{Vec}_{n,T}^{-1}((W^{i,k})_{j,:})
$, so that
$(W^{i,k})_{j,:} X_{\mathrm{Vec}} = \sum_{t=1}^T \bm{x}_t^\top (V^{(i,j)})_{:,t}.
$
By Lemma~\ref{lem:mlp_piece}, each term $\bm{x}_t^\top (V^{(i,j)})_{:,t}$ can be
implemented by a token-wise feedforward layer after partitioning the input
tokens into distinct regions. 
Concretely, the feedforward layer realizes a piecewise linear function where the
$t$-th region corresponds to the linear map
$\bm{x}_t^\top (V^{(i,j)})_{:,t}$ for $t \in [T]$.

The resulting scalars
$\{\bm{x}_t^\top (V^{(i,j)})_{:,t}\}_{t=1}^T$ are then aggregated via a
self-attention layer, which is the only Transformer component that enables
inter-token interactions.

\paragraph{Step 2: Approximating the max operation.}
Each term $\max(W^{i,k} X_{\mathrm{Vec}} + \bm{b}^{i,k})$ is approximated by
an additional self-attention layer. 
This relies on the observation
\[
\bm{x}^\top \sigma_S^\lambda(\bm{x}) \approx \bm{x}^\top \sigma_H(\bm{x}) = \max(\bm{x})
\]
for sufficiently large $\lambda > 0$ (see Lemma~\ref{lem:softmax_app_max}),
so that the approximation error can be made arbitrarily small by taking $\lambda$ large enough.

\begin{remark}\label{remark_ para}
In Theorem~\ref{thm_shallow_max}, a maxout layer with parameter complexity $\mathcal{O}(T^2n m p)$ can be approximated 
by a Transformer with the same order of parameters, once the sparsity of both the 
feedforward and self-attention layers is taken into account 
(see Remark~\ref{remark:para_complex_lem_appen}). 
This shows that the Transformer-based approximation is parameter-efficient 
relative to the maxout baseline.
\end{remark}

\begin{remark}\label{p_ge_T}
Theorem~\ref{thm_shallow_max} assumes $p \le T$, exploiting
$\mathbf{x}^\top \sigma^\lambda(\mathbf{x}) \approx \sigma_H(\mathbf{x}) = \max(\mathbf{x})$
for large $\lambda$. 
For $p > T$, the same computation can be implemented by a rank-$s$ maxout network
of depth $\lceil (p-1)/(s-1) \rceil$ for any $s \le T$, equipped with residual
connections from the input. 
This regime corresponds to the approximation of deep maxout networks by Transformer
networks, which will be addressed in Theorem~\ref{thm_maxout_p>T}.
\end{remark}

\begin{remark}\label{remark:Lip_NN}
In Theorem~\ref{thm_shallow_max}, if $\lambda \ge C(T,n,M_1,M_2)$, the approximation error between the
hardmax-based network $\mathcal{N}_H$ and the corresponding softmax-based network
$\mathcal{N}_S^\lambda$ satisfies (see Remark~\ref{remark:Lip_NN_lem} for details)
\[
\|\mathcal{N}_H - \mathcal{N}_S^\lambda\|_{L^\infty(\Omega)} \le C_1 T^2 \lambda^{-1},
\]
where $C_1$ is an absolute constant, and
\[
M_1 := \max_{X \in \Omega} \|X\|_\infty, 
\quad
M_2 := \max_{i \in [m], k \in [T]} \max\{\|\bm{b}^{i,k}\|_\infty, \|W^{i,k}\|_\infty\}.
\]

Moreover, $\mathcal{N}_H$ is Lipschitz continuous on $\Omega$, with
\[
\mathrm{Lip}(\mathcal{N}_H; \Omega) \le C_2 p T^2 M_2,
\]
for some absolute constant $C_2$.
\end{remark}

By Theorem~\ref{thm_shallow_max}, a three-layer Transformer network can
approximate a single maxout layer. 
We now extend this construction to approximate a depth-$D$ maxout network
by sequentially stacking $D$ such Transformer subnetworks.

\begin{thm}[Approximation of Deep Maxout Networks with $p \le T$]\label{thm_deep_maxout}
Let $p \le T$, let $\Omega \subset \mathbb{R}^{n \times T}$ be compact, and
let $f : \mathbb{R}^{n \times T} \to \mathbb{R}^{m \times T}$ be realized
by a $D$-layer maxout network with input dimension $nT$, width $mT$ at each
layer, and rank $p$. 
Then there exists a hardmax-based Transformer network
\[
\begin{aligned}
\mathcal{N}_H \in \mathcal{TFN}^{\mathrm{hard}}_{n,m}\bigl(
&L = 3D, \\
&d = \max\{n,\, m[(T+1)(p+1)+1]\} + T + 1, \\
&k = 2, \\
&H = (T+1) m p, \\
&r = 4(\max\{n,mp\}+1)(T+3)
\bigr),
\end{aligned}
\]
such that
\[
\mathcal{N}_H\big|_{\Omega} = f\big|_{\Omega}.
\]

Moreover, for any $\epsilon > 0$, the corresponding softmax-based Transformer
$\mathcal{N}_S^\lambda$ satisfies
\[
\|\mathcal{N}_S^\lambda - f\|_{L^\infty(\Omega)} < \epsilon,
\]
provided that the scaling parameter satisfies $\lambda = \mathcal{O}(1/\epsilon)$.
\end{thm}

The construction of Theorem~\ref{thm_deep_maxout} builds on Theorem~\ref{thm_shallow_max} and relies on two key ideas. 
First, as in Theorem~\ref{thm_shallow_max},
positional embeddings are used to mitigate the effect of parameter sharing across tokens in the feedforward layers,
allowing each feedforward module to approximate a piecewise-linear function.
For the multi-layer setting, we further introduce token-dependent shifts at each layer,
which move token representations into pairwise disjoint regions and preserve the maxout computation.
Second, by Remark~\ref{remark:Lip_NN}, each hardmax-based subnetwork admits a uniform Lipschitz bound,
and the approximation error between a hardmax subnetwork and its softmax counterpart scales as $O(1/\lambda)$.
As a result, stacking $D$ such subnetworks yields a deep Transformer network whose softmax version
approximates the corresponding deep maxout network with controlled error accumulation.

\begin{remark}
If the scaling parameter satisfies
$
\lambda \ge C\!\Bigl(T, n, M, D, (m^\ell)_{\ell=1}^D\Bigr),
$
then the approximation error between the hardmax-based Transformer
$\mathcal{N}_H$ and its softmax-based counterpart $\mathcal{N}_S^\lambda$
satisfies
\[
\|\mathcal{N}_H - \mathcal{N}_S^\lambda\|_{L^\infty(\Omega)}
\le C\!\Bigl(T, M, D, (m^\ell)_{\ell=1}^D\Bigr) \lambda^{-1}.
\]

Here,
\[
M := \sup_{X \in \Omega} \|X\|_\infty, 
\qquad
m_\ell := \max_{i \in [m], j \in [T]}\max 
\{\|W^\ell_{i,j}\|_\infty, \|\bm{b}^\ell_{i,j}\|_\infty\},
\]
where $\{W^\ell_{i,j}\}_{i,j}$ and $\{\bm{b}^\ell_{i,j}\}_{i,j}$ denote the
weight matrices and bias vectors of the $\ell$-th maxout layer, respectively.
\end{remark}

For a feedforward neural network defined as \eqref{eq: FFNN} with the ReLU activation function $x \mapsto \max\{0,x\}$, we refer to it as a ReLU network.  
Such a network can be regarded as a special case of a rank-$2$ maxout network followed by an affine readout layer.  
Hence, Theorem~\ref{thm_deep_maxout} immediately implies that ReLU networks admit universal approximation by Transformer networks, under vectorized representations of inputs and outputs.  
Moreover, the affine readout can be implemented by an additional Transformer layer, since a token-wise feedforward layer combined with self-attention can approximate affine mappings (cf.~Step~1 in Proof Sketch of Theorem~\ref{thm_shallow_max}).

\begin{coro}[Universal Approximation of ReLU Networks]
Let $T\ge 2$, $\Omega \subset \mathbb{R}^{n \times T}$ be compact, and let \( f : \br^{n \times T} \to \mathbb{R}^{m \times T} \) be computed by a \( D \)-layer ReLU network with width \( mT \) per layer. Then, there exists a hardmax-based Transformer network
\[
\begin{aligned}
\mathcal{N}_H \in \mathcal{TFN}^{\mathrm{hard}}_{n,m}\bigl(
&L = 3D+1, \\
&d =\max\{n,m[3(T+1)+1]\}+T+1, \\
&k = 2, \\
&H =2(T+1)m, \\
&r = 4(\max\{n,2m\}+1)(T+3)
\bigr),
\end{aligned}
\]
such that 
\[
\mathcal{N}_H\big|_{\Omega}= f\big|_{\Omega}.
\]
Moreover, for any \( \epsilon > 0 \), the corresponding softmax-based Transformer $\mathcal{N}_S^\lambda$ satisfies 
\[
\|\mathcal{N}_S^\lambda- f\|_{L^\infty(\Omega)} < \epsilon,
\]
provided that the scaling parameter satisfies $\lambda=\mathcal{O}(1/\epsilon)$.
\end{coro}

As discussed in Remark~\ref{p_ge_T}, a shallow maxout network of rank $p > T \ge 2$ can be realized by a rank-$s$ maxout network of depth $\lceil \tfrac{p-1}{s-1} \rceil$ for any $2 \le s \le T$, equipped with residual connections from the input. 
Combining this observation with Theorem~\ref{thm_deep_maxout}, we obtain the following result:

\begin{thm}[Universal Approximation of Shallow Maxout Networks]\label{thm_maxout_p>T}
Let $p,T \ge 2$, let $\Omega \subset \mathbb{R}^{n \times T}$ be compact, and let 
$f \in \mathcal{T}_{\mathrm{max}}(n \times T, p, m \times T)$. 
Then, for any $2 \le s \le T$, there exists a hardmax-based Transformer network
\[
\begin{aligned}
\mathcal{N}_H \in \mathcal{TFN}^{\mathrm{hard}}_{n,m}\bigl(
&L = 3 \Bigl\lceil \frac{p-1}{s-1} \Bigr\rceil, \\
&d = d'+T+1, \\
&k = 2, \\
&H = (T+1)m \min\{s,p\}, \\
&r = 4(r'+1)(T+3)
\bigr),
\end{aligned}
\]
with
\[
d' =
\begin{cases} 
\max\{n, m(Tp +T+1)\}, & s \ge p,\\[2mm]
n + m[(T+1)(s+1)+1], & s < p,
\end{cases}
\qquad
r' =
\begin{cases} 
n, & s \ge p,\\[2mm]
\max\{n+m,ms\}, & s < p,
\end{cases}
\]
such that
\[
\mathcal{N}_H\big|_\Omega = f\big|_\Omega.
\]

Moreover, for any $\epsilon > 0$, the corresponding softmax-based Transformer 
$\mathcal{N}_S^\lambda$ satisfies
\[
\|\mathcal{N}_S^\lambda - f\|_{L^\infty(\Omega)} < \epsilon,
\]
provided that the scaling parameter satisfies $\lambda = \mathcal{O}(1/\epsilon)$.
\end{thm}

To approximate a deep maxout network with $D$ layers, we sequentially stack $D$ Transformer subnetworks, each constructed as in Theorem~\ref{thm_maxout_p>T}.  
Combining Theorems~\ref{thm_deep_maxout} and~\ref{thm_maxout_p>T} yields the following universal approximation result.

\begin{thm}[Universal Approximation of Deep Maxout Networks]
\label{thm_deep_maxout_p>T}
Let $p,T\ge 2$, let $\Omega \subset \mathbb{R}^{n \times T}$ be compact, and let
$f : \mathbb{R}^{n \times T} \to \mathbb{R}^{m \times T}$ be computed by a
$D$-layer maxout network with input dimension $dT$, width $mT$ per layer,
and rank $p$. For any $2 \le s \le T$, there exists a hardmax-based Transformer
network
\[
\begin{aligned}
\mathcal{N}_H \in \mathcal{TFN}^{\mathrm{hard}}_{n,m}\bigl(
&L = 3 \Bigl\lceil \frac{p-1}{s-1} \Bigr\rceil D, \\
&d = d' + T + 1, \\
&k = 2, \\
&H = (T+1)m \min\{s,p\}, \\
&r = 4(r' + 1)(T+3)
\bigr),
\end{aligned}
\]
with
\[
d' =
\begin{cases} 
\max\{n, m[(T+1)(p + 1)+1]\}, & s \ge p,\\[1mm]
\max\{n,m\} + m[(T+1)(s + 1)+1], & s < p,
\end{cases}
\qquad
r' =
\begin{cases} 
\max\{n,mp\}, & s \ge p,\\[1mm]
\max\{n+m,ms\}, & s < p,
\end{cases}
\]
such that
\[
\mathcal{N}_H\big|_{\Omega} = f\big|_{\Omega}.
\]
Moreover, for any $\epsilon > 0$, the corresponding softmax-based Transformer
$\mathcal{N}_S^\lambda$ satisfies
\[
\|\mathcal{N}_S^\lambda - f\|_{L^\infty(\Omega)} < \epsilon,
\]
provided that the scaling parameter satisfies $\lambda = \mathcal{O}(1/\epsilon)$.
\end{thm}

\begin{remark}
When $s \ge p$, Theorems~\ref{thm_maxout_p>T} and~\ref{thm_deep_maxout_p>T}
reduce to Theorems~\ref{thm_shallow_max} and~\ref{thm_deep_maxout},
respectively.
\end{remark}

\begin{remark}
For the Transformer networks in Theorems~\ref{thm_deep_maxout}, 
\ref{thm_maxout_p>T}, and \ref{thm_deep_maxout_p>T}, the total number of 
parameters matches that of the target maxout network up to constants when the sparsity 
of both feedforward and self-attention layers is considered, achieving 
parameter-efficient approximation (see Remark~\ref{remark_ para}). 
This follows from sequentially stacking subnetworks that approximate individual maxout layers.
\end{remark}

According to~\cite{balazs2015near}, any convex and Lipschitz continuous function can be approximated by the maximum
of finitely many affine functions; see Lemma~\ref{lem:convex_max}. Consequently, any such function can be approximated by
a maxout layer with sufficiently large rank. Combining this observation with Theorem~\ref{thm_maxout_p>T}, we obtain the following result.

\begin{coro}[Approximation of Convex Lipschitz Functions]
\label{thm_convex_Lipsc}
Let $T \ge 2$ and let $\Omega \subset \mathbb{R}^{n \times T}$ be a compact convex
set with nonempty interior. Suppose that
$f \colon \Omega \to \mathbb{R}^{m \times T}$ is convex and Lipschitz continuous
with Lipschitz constant $C$.
Then, for any integer $p \ge 2$ and any $2 \le s \le T$, there exists a
hardmax-based Transformer network
\[
\begin{aligned}
\mathcal{N}_H \in \mathcal{TFN}^{\mathrm{hard}}_{n,m}\bigl(
&L = 3 \Bigl\lceil \frac{p-1}{s-1} \Bigr\rceil, \\
&d = d'+T+1, \\
&k = 2, \\
&H = (T+1)m \min\{s,p\}, \\
&r = 4(r'+1)(T+3)
\bigr),
\end{aligned}
\]
where
\[
d' =
\begin{cases} 
\max\{n, m(Tp +T+1)\}, & s \ge p,\\[2mm]
n + m[(T+1)(s+1)+1], & s < p,
\end{cases}
\qquad
r' =
\begin{cases} 
n, & s \ge p,\\[2mm]
\max\{n+m,ms\}, & s < p.
\end{cases}
\]
The network $\mathcal{N}_H$ satisfies the uniform approximation bound
\[
\|\mathcal{N}_H - f\|_{L^\infty(\Omega)}
\le
72\, n T^2 C\, \mathrm{diam}(\Omega)\, p^{-2/(nT)},
\]
where
$
\mathrm{diam}(\Omega)
:= \sup_{\bfx,\bfy \in \Omega} \|\bfx - \bfy\|_\infty .
$

Moreover, for any $\epsilon > 0$, the corresponding softmax-based Transformer
$\mathcal{N}_S^\lambda$ satisfies
\[
\|\mathcal{N}_S^\lambda - f\|_{L^\infty(\Omega)}
<
72\, n T^2 C\, \mathrm{diam}(\Omega)\, p^{-2/(nT)} + \epsilon,
\]
provided that the scaling parameter satisfies
$\lambda = \mathcal{O}(1/\epsilon)$.
\end{coro}

\section{The Number of Linear Regions for Transformer Networks}
\label{sec:CPWL}

Building on Section~\ref{sec: appro_maxout}, where Transformer networks were shown
to approximate maxout networks, we extend the analysis to a broader and
fundamental function class, namely continuous piecewise linear (CPWL) functions.
A natural way to quantify the expressive power of Transformer architectures is
through the number of linear regions they can represent.
Accordingly, we study the architectural complexity required for Transformer
networks to realize CPWL functions with a prescribed number of linear regions,
and conversely, characterize the maximal number of linear regions achievable by
Transformers of a given complexity.

Classically, CPWL functions are defined as mappings
$f:\mathbb{R}^{n}\to\mathbb{R}^{m}$.
Here we extend this notion to sequence-to-sequence mappings
$f:\mathbb{R}^{n\times T}\to\mathbb{R}^{m\times T}$ by identifying inputs and
outputs with their vectorized representations.
In particular, when $T=1$, this definition reduces to the standard one, and the
two notions coincide.

We begin by formalizing CPWL functions in the sequence-to-sequence setting.

\begin{de}[continuous piecewise linear (CPWL) functions]
A function $f:\mathbb{R}^{n\times T}\to\mathbb{R}^{m\times T}$ is called
CPWL if it is continuous and if there exists
a finite collection of closed and nonempty sets
$\{\mathcal{I}_k\}_{k=1}^K \subset \mathbb{R}^{n\times T}$ with nonempty, pairwise disjoint interiors and
$\bigcup_{k=1}^K \mathcal{I}_k = \mathbb{R}^{n\times T}$, such that $f$ is affine on
each $\mathcal{I}_k$.
That is, for every $k=1,\dots,K$, there exist
$W^k\in\mathbb{R}^{mT\times nT}$ and $\bm{b}^k\in\mathbb{R}^{mT}$ such that
$
f(X)
=
\mathrm{Vec}^{-1}_{m,T}\!\left(W^k\,\mathrm{Vec}(X)+\bm{b}^k\right)$, $X\in\mathcal{I}_k$.

\end{de}
\begin{de}[Linear regions of CPWL functions]
Let $f:\mathbb{R}^{n\times T}\to\mathbb{R}^{m\times T}$ be CPWL.
A set $\mathcal{I}\subset\mathbb{R}^{n\times T}$ is called a \emph{linear region}
of $f$ if it is a maximal connected subset on which $f$ is affine.
We denote by $N(f)$ the number of linear regions of $f$, and for a class of CPWL
functions $\mathcal{F}$ define
$
N(\mathcal{F}) := \sup_{f\in\mathcal{F}} N(f).
$
\end{de}

We now relate CPWL functions to maxout representations.
By Lemma~\ref{lem:cpwl_convex}, any CPWL function can be decomposed as the
difference of two convex CPWL functions.
Moreover, each convex CPWL function can be represented as the maximum of finitely
many affine functions.
Combining this decomposition with the universal approximation results for
maxout networks established in Theorem~\ref{thm_maxout_p>T}, we conclude that
Transformer networks are capable of representing arbitrary CPWL functions.
This is formalized in Theorem~\ref{thm:tf_CPWL} below.

\begin{lem}[{\cite[Proposition~4.3]{Hertrich2023SIAM}}]\label{lem:cpwl_convex}
Let $f:\mathbb{R}^n\to\mathbb{R}$ be a CPWL function with $N$ linear regions.
Then $f$ admits a decomposition $f=g-h$, where both $g$ and $h$ are convex CPWL
functions with at most $N^{2n+1}$ linear regions.
\end{lem}

\begin{thm}[Universal Approximation of CPWL Functions]\label{thm:tf_CPWL}
Let $N,T\ge 2$, and let
$f:\mathbb{R}^{n\times T}\to\mathbb{R}^{m\times T}$
be a CPWL function with at most $N$ linear regions.
Then, for any compact set $\Omega\subset\mathbb{R}^{n\times T}$, there exists a
hardmax-based Transformer network
\[
\begin{aligned}
\mathcal{N}_H \in \mathcal{TFN}^{\mathrm{hard}}_{n,m}\bigl(
&L = 3 \Bigl\lceil \frac{N^{2nT+1}-1}{T-1} \Bigr\rceil, \\
&d = n+2m(T+1)^2+2m+T+1, \\
&k = 2, \\
&H = 2mT(T+1), \\
&r = 4(\max\{n+2m,2mT\}+1)(T+3)
\bigr),
\end{aligned}
\]
such that
\[
\mathcal{N}_H\big|_{\Omega} = f\big|_{\Omega}.
\]
Moreover, for any $\epsilon>0$, the corresponding softmax-based Transformer
$\mathcal{N}_S^\lambda$ satisfies
\[
\|\mathcal{N}_S^\lambda - f\|_{L^\infty(\Omega)} < \epsilon,
\]
provided that the scaling parameter satisfies $\lambda=\mathcal{O}(1/\epsilon)$.
\end{thm}

Next, we study the maximal number of linear regions that can be attained by
CPWL functions approximable by Transformer networks with a fixed architecture.
To this end, we extend the notion of linear regions from CPWL functions to
Transformer architectures.
Specifically, by Theorem~\ref{thm:tf_CPWL}, although Transformer networks are not
CPWL functions themselves, any CPWL function with an arbitrarily prescribed
number of linear regions can be approximated, on an arbitrary compact set and to
arbitrary accuracy, by a Transformer network of fixed architecture through a
suitable choice of parameters.

Let $\mathcal{F}$ denote the class of functions realizable by Transformer
networks with a fixed architecture, and let $f$ be a CPWL function.
If, for any compact set $\Omega \subset \mathbb{R}^{n\times T}$ and any
$\epsilon>0$, there exists $\hat f \in \mathcal{F}$ such that
\[
\|\hat f - f\|_{L^\infty(\Omega)} < \epsilon,
\]
then we say that the Transformer architecture can realize at least $N(f)$
linear regions, and we write
\[
N(\mathcal{F}) \ge N(f).
\]

The following lemma provides a lower bound on the number of linear regions
realizable by deep maxout networks \cite{MontufarSiam2022}.

\begin{lem}[{\cite[Proposition~3.12]{MontufarSiam2022}}]
\label{lem_linear_region_maxout}
Consider a rank-$k$ maxout network $\mathcal{N}$ with $n_0$ inputs and $L$ layers, with layer widths $n_1,\dots,n_L$.
Let
$
n \le n_0,\ \tfrac{n_1}{2},\dots,\tfrac{n_{L-1}}{2},
$
and assume that $n_l/n$ is even for $l=1,\dots,L-1$
(otherwise, take the largest even lower bound and discard the rest).
Then the number of linear regions satisfies
\[
N(\mathcal{N})
\ge
\Biggl(
\prod_{l=1}^{L-1}
\bigl(\tfrac{n_l}{n}(k-1)+1\bigr)^n
\Biggr)
\Biggl(
\sum_{j=0}^{n} \binom{n_L}{j}(k-1)^j
\Biggr).
\]
\end{lem}

Combining Lemma~\ref{lem_linear_region_maxout} with the universal approximation
of deep maxout networks by Transformer networks established in
Theorem~\ref{thm_deep_maxout_p>T}, we obtain a lower bound on the maximal number
of linear regions attainable by Transformer networks with a fixed architecture.

\begin{thm}[Number of linear regions for Transformer networks]
\label{thm:num_linear_transformer}
Let $\mathcal{F}$ denote the class of functions realizable by Transformer
networks with input dimension $n\times T$, output dimension $m\times T$,
and size of
\[
\begin{aligned}
\bigl(
&L = D,\\
&d = \max\{n,\, m[(T+1)^2+1]\}+T+1,\\
&k = 2,\\
&H = (T+1)Tm,\\
&r = 4(\max\{n,mT\}+1)(T+3)
\bigr),
\end{aligned}
\]
where $D\ge 3$.
Then
\begin{equation}\label{eq:num_linear_transformer}
N(\mathcal{F})
\ge
\Bigl[\tfrac{mT}{q}(T-1)+1\Bigr]^{\,q(\lfloor D/3\rfloor-1)}
\sum_{j=0}^{q}\binom{mT}{j}(T-1)^j,
\end{equation}
for every integer $q$ satisfying
$
q \le \min\{nT,\, mT/2\}$, and
$\tfrac{mT}{q}$ is even.
\end{thm}

\begin{remark}
Theorem~\ref{thm:num_linear_transformer} shows that the number of linear regions
attainable by Transformer networks grows exponentially with the network depth~$D$,
highlighting the expressive power gained through depth.
\end{remark}

\section{Conclusions}\label{sec:conclusion}
This paper studies the expressive power of Transformer networks.
We construct Transformer architectures that uniformly approximate maxout
networks with comparable model complexity, thereby establishing a direct link
between the approximation theory of feedforward neural networks and Transformers.
Building on this connection, we develop a principled framework for analyzing the
ability of Transformers to approximate continuous piecewise linear (CPWL)
functions and provide a quantitative characterization of their expressivity in
terms of the number of linear regions.

Several directions for future research naturally arise from this work.
One direction is to transfer refined approximation results for feedforward neural
networks—such as approximation rates on specific function spaces and techniques for
alleviating the curse of dimensionality—to Transformer models via their systematic
approximation of maxout and ReLU networks.
Another important question is whether pure self-attention architectures can
efficiently emulate standard feedforward networks, and more broadly, whether
Transformers or self-attention–only models can surpass feedforward networks in
expressive power, for instance as measured by the growth of linear regions.

\begin{appendices}\label{sec:appendix}
%\section{Proofs}\label{sec:prf}
\section{Proof of Theorem \ref{thm_shallow_max}}
To prove Theorem~\ref{thm_shallow_max}, we first consider a simplified setting in which
a Transformer network approximates a function
$f \in \mathcal{T}_{\mathrm{max}}(n \times T, p, m)$ at a single token position,
while suppressing the outputs at all other positions.
This case is established in the following lemma.
The general sequence-to-sequence result then follows by a straightforward extension with the same architecture is applied  
across $T$ positions.

\begin{lem}\label{lem_shallow_max_token}
Let $p\le T$, $\Omega \subset \mathbb{R}^{n \times T}$ be a compact set, and let $f \in \mathcal{T}_{\text{max}}(n \times T, p, m)$.
Then, for any $k \in [T]$, there exists a hardmax-based Transformer network
\[
\begin{aligned}
\mathcal{N}_H \in \mathcal{TFN}^{\mathrm{hard}}_{n,m}\bigl(
&L = 3, \\
&d = \max\{n,\, mp+2m\}+T+1, \\
&k = 2, \\
&H = mp, \\
&r = 4(n+1)(T+3)
\bigr),
\end{aligned}
\]
such that 
\[
(\mathcal{N}_H)_{k}\big|_{\Omega} = f\big|_{\Omega} , \quad \text{and} \quad (\mathcal{N}_H)_{k'}\big|_{\Omega} =\bm{0}, \quad \text{for all } k' \in [T] \setminus \{k\}.
\]
Moreover, for any $\epsilon>0$, the corresponding softmax-based Transformer
$\mathcal{N}_S^\lambda$ satisfies
\[
\|(\mathcal{N}_S^\lambda)_{k} - f\|_{L^\infty(\Omega)} < \epsilon, \quad \text{and} \quad \|(\mathcal{N}_S^\lambda)_{k'}\|_{L^\infty(\Omega)} < \epsilon \quad \text{for all } k' \in [T] \setminus \{k\},
\]
provided that the scaling parameter satisfies $\lambda=\mathcal{O}(1/\epsilon)$.
\end{lem}

{\bf{Proof.}}
Let $X = (\bfx_1, \dots, \bfx_T) \in \mathbb{R}^{n \times T}$ and denote by
$X_{\mathrm{Vec}} := \mathrm{Vec}(X) \in \mathbb{R}^{nT}$ its vectorized form.
By definition, any function
$f \in \mathcal{T}_{\mathrm{max}}(n \times T, p, m)$
can be written as
\[
f(X)
=
\bigl[
\max(W^1 X_{\mathrm{Vec}} + \bm{b}^1), \dots,
\max(W^m X_{\mathrm{Vec}} + \bm{b}^m)
\bigr]^\top,
\]
where $W^i \in \mathbb{R}^{p \times nT}$ and
$\bm{b}^i = (b^i_j)_{j=1}^p \in \mathbb{R}^p$ for $i \in [m]$.

For each $i \in [m]$ and $j \in [p]$, define
\[
V^{(i,j)} := \mathrm{Vec}^{-1}_{n,T}\!\bigl((W^i)_{j,:}\bigr) \in \mathbb{R}^{n \times T}.
\]
Then, for any $X \in \mathbb{R}^{n \times T}$,
\[
(W^i)_{j,:} X_{\mathrm{Vec}} + b^i_j
=
\sum_{t=1}^T \bfx_t^\top (V^{(i,j)})_{:,t} + b^i_j,
\]
which expresses each affine component of $f$ as a sum of token-wise inner products.

We first construct a hardmax-based Transformer network that realizes the desired
token-selective approximation.
The approximation result for the softmax-based Transformer then follows from
a quantitative comparison between the hardmax and softmax activations.

\medskip
\noindent
\textbf{(1) Hardmax Case}

We construct the network in two main steps.

\smallskip
\noindent
\textit{Step 1: Affine map via feedforward layer and self-attention layer.}

A feedforward layer and a self-attention layer are used to approximate $\{W^iX_{\text{Vec}}+\bm{b}^i\}_{i=1}^m$.
Formally, together with position embedding, the resulting mapping sends $X\in \Omega$ to 
\begin{equation}\label{eq: Z}
Z=\begin{pmatrix} Z_1\\Z_2\\\vdots\\Z_m\end{pmatrix}, ~Z_i=\begin{pmatrix}(W^iX_\text{Vec}+\bm{b}^i)^\top &\bm{0}\end{pmatrix}, {\text{~for~}} i\in [m].
\end{equation}

\begin{itemize}
\item 
Position embeddings are introduced to map each token of the input sequence into a distinct domain.

Denote $$
\begin{aligned}
a := \min\left\{ x_{ij} \;\middle|\; X = (x_{ij}) \in \Omega,\ i \in[n],\ j \in [T] \right\},\\
b:=\max\left\{ x_{ij} \;\middle|\; X = (x_{ij}) \in \Omega,\ i \in[n],\ j \in [T] \right\},
\end{aligned}
$$
then $\Omega\subset [a,b]^{n\times T}$.
We use the following position embedding:
$$\tilde{X}=\begin{pmatrix}
\begin{pmatrix} X&\bm{0}_n\end{pmatrix}+P\\
\bm{0}\\
I_{T+1}
\end{pmatrix}\in \mathbb{R}^{d\times (T+1)}, \text{~with~} P:=\tilde{\bm{p}} \bm{1}_{n}^\top,
$$ 
where $\tilde{\bm{p}}=(\tilde{p}_t)_t\in \br^{T+1}$, $\tilde{p}_t=(b-a)(t-1)+t\delta$ for $t\in [T]$, and $\tilde{p}_{T+1}=a+(b-a)(T+1/2)+(T+1)\delta$, with $0<\delta<(b-a)/(T+1)$. Denoting $X=(\bfx_1,\cdots,\bfx_T)\in\Omega$ and $\tilde{X}=(\tilde{\bfx}_1,\cdots,\tilde{\bfx}_T,\tilde{\bfx}_{T+1})$, it is easy to see that $(\tilde{\bfx}_t)_{1:n}\in \mathcal{I}_{t}:= [a,b]^{n}+((t-1)(b-a)+t\delta)\bm{1}_n$ for $t\in [T+1]$.
Thus, $\mathcal{I}_{i}\cap \mathcal{I}_{j}=\emptyset$ for any $i, j\in [T+1]$, $i\neq j$.

\item
A feedforward layer is employed to realize token-wise affine mappings $\{\bfx_t^\top (V^{(i,j)})_{:,t}+b^i_j\}$.

Formally, it maps $\tilde X$ to
\[
\tilde Y
=
\begin{pmatrix}
Y\\
\bm{0}\\
I_{T+1}
\end{pmatrix}
\in \mathbb{R}^{d\times (T+1)},
\qquad
Y=(\bfy_1,\ldots,\bfy_T,\bfy_{T+1})\in \mathbb{R}^{mp\times (T+1)} .
\]
For each $t\in[T]$, the output token $\bfy_t\in\mathbb{R}^{mp}$ is defined componentwise by
\[
(\bfy_t)_{(i-1)p+j}
=
\bfx_t^\top (V^{(i,j)})_{:,t},
\qquad i\in[m],\ j\in[p],
\]
that is, $\bfy_t$ stacks all inner products between $\bfx_t$ and the $t$-th columns of
$\{V^{(i,j)}\}_{i\in[m],\,j\in[p]}$.

And the last token is given by
\[
\bfy_{T+1}=\bm{0}_{mp}.
\]

For simplicity of presentation, we omit the bias terms $\bm{b}^i$; they can be
incorporated into the first output token $\bfy_1$ by replacing
$\bfx_1^\top (V^{(i,j)})_{:,1}$ with
$\bfx_1^\top (V^{(i,j)})_{:,1} + b^i_j$.

Since the positional embedding maps each token into a distinct domain, it suffices
to construct each output neuron as a piecewise linear function that coincides with
$\bfx_t^\top (V^{(i,j)})_{:,t}$ on $\mathcal{I}_t$ for $t \in [T]$, and with $0$ on
$\mathcal{I}_{T+1}$.
By Lemma~\ref{lem:mlp_piece}, such a function can be exactly represented by a
one-layer feedforward network with $4(n+1)(T+2)$ hidden neurons.

To compensate for the residual connection, we additionally use the identity
$\mathrm{ReLU}(-x) - \mathrm{ReLU}(x) = -x$.
Thus, a feedforward layer with $2n$ hidden neurons maps
$\begin{pmatrix} X & \bm{0}_n \end{pmatrix} + P$
to
$-\begin{pmatrix} X & \bm{0}_n \end{pmatrix} - P$.

Altogether, a residual feedforward layer with
$4(n+1)(T+2) + 2n$ hidden neurons maps $\tilde{X}$ to $\tilde{Y}$.

\item 
A self-attention layer is employed to aggregate the tokens to obtain $\{W^iX_{\text{Vec}}+\bm{b}^i\}_{i=1}^m$, leveraging the attention mechanism's ability to calculate a weighted sum of the input tokens.

Formally, the self-attention layer consists of $mp$ heads, each of size 2, and can be written as
$$\tilde{Y}+\sum_{i=1}^m\sum_{j=1}^p W_O^{(i)}W_V^{(i,j)}\tilde{Y}\sigma[(W_K^{(j)}\tilde{Y})^\top W_Q^{(j)}\tilde{Y}],$$
where $W^{(i,j)}_K,~W^{(i,j)}_Q,~W^{(i,j)}_V\in \br^{2\times d}$, $W^{(i,j)}_O\in \br^{d\times 2}$, and each head corresponds to $(W^i)_{j,:}X_{\text{Vec}}$.

We first specify the value matrices. Let ${W_V^{(i,j)}}\tilde{Y}$ extract the $(p(i-1)+j)$-th row of $\tilde{Y}$; that is 
\begin{equation}\label{eq:V_mat}
{W_V^{(i,j)}}\tilde{Y}=
\begin{pmatrix}
T(\bfx_1^\top {(V^{(i,j)}})_{:,1}+b^i_j)&T\bfx_ 2^\top ({V^{(i,j)}})_{:,2}&\cdots&T\bfx_ T^\top ({V^{(i,j)}})_{:,T}&0\\
0&0&\cdots&0&0
\end{pmatrix}
,
\end{equation}
which is achieved by setting the $(p(i-1)+j)$-th column of ${W_V^{(i,j)}}$ to $(T,0)^\top$ and all other column to $\bm{0}$.

To obtain $(W^i)_{j,:}X_{\text{Vec}}+b^i_j$, i.e., $\sum_{t=1}^T\bfx_t^\top ({V^{(i,j)}})_{:,t}+b^i_j$, we aim to make
\begin{equation}\label{eq:att_mat}
\sigma_H[(W_K^{(j)}\tilde{Y})^\top W_Q^{(j)}\tilde{Y}]=
\bordermatrix{
&1&\cdots&j-1&j&j+1&\cdots&T+1\cr
& 0&\cdots&0&\frac{1}{T}&0&\cdots&0\cr
&\vdots&\ddots&\vdots&\vdots&\vdots&\ddots&\vdots\cr
&0&\cdots&0&\frac{1}{T}&0&\cdots&0\cr
 &1&\cdots&1&0&1&\cdots&1
}.
\end{equation}
Assuming \eqref{eq:att_mat} holds, we obtain
$$
W_V^{(i,j)}\tilde{Y}\sigma_H[(W_K^{(j)}\tilde{Y})^\top W_Q^{(j)}\tilde{Y}]=
\bordermatrix{
&1&\cdots&j-1&j&j+1&\cdots&T+1\cr
& 0&\cdots&0&(W^i)_{j,:}X_{\text{Vec}}+b^i_j&0&\cdots&0\cr
 &0&\cdots&0&0&0&\cdots&0
}.
$$
To realize \eqref{eq:att_mat}, it suffices to choose
\[
W_K^{(j)}\tilde{Y}
=
\bordermatrix{
& 1 & \cdots & T & T+1 \cr
& \bm{v}_1 & \cdots & \bm{v}_1 & \bm{v}_2
},
\]
and
\[
W_Q^{(j)}\tilde{Y}
=
\bordermatrix{
& 1 & \cdots & j-1 & j & j+1 & \cdots & T+1 \cr
& \bm{v}_2 & \cdots & \bm{v}_2 & \bm{v}_1 & \bm{v}_2 & \cdots & \bm{v}_2
},
\]
where $\bm{v}_1=(1,0)^\top$ and $\bm{v}_2=(0,1)^\top$.
This construction is feasible since, by choosing
$
W_K^{(j)}=\begin{pmatrix}\bm{0} & A\end{pmatrix}$, $A\in\mathbb{R}^{2\times(T+1)},
$
and noting that $I_{T+1}$ is included in $\tilde{Y}$,
$W_K^{(j)}\tilde{Y}$ can realize any matrix in $\mathbb{R}^{2\times(T+1)}$ through an appropriate choice of $A$.
The same argument applies to $W_Q^{(j)}\tilde{Y}$.

Consequently,
\[
\sum_{j=1}^p W_V^{(i,j)} \tilde{Y}\,
\sigma_H\!\bigl[(W_K^{(j)}\tilde{Y})^\top W_Q^{(j)}\tilde{Y}\bigr]
=
\begin{pmatrix}
(W^i)_{1,:}X_{\mathrm{Vec}}+b^i_1 & \cdots &
(W^i)_{p,:}X_{\mathrm{Vec}}+b^i_p & \bm{0}\\
\bm{0} & \cdots & \bm{0} & \bm{0}
\end{pmatrix}
=
\begin{pmatrix}
Z_i\\
\bm{0}
\end{pmatrix}.
\]

Next, define the output projection $W_O^{(i)}$ by
\[
(W_O^{(i)})_{mp+i,:}=(1,0), \qquad
(W_O^{(i)})_{i',:}=(0,0)\ \text{for } i'\neq mp+i .
\]
It follows that
\[
\sum_{i=1}^m\sum_{j=1}^p
W_O^{(i)} W_V^{(i,j)} \tilde{Y}\,
\sigma_H\!\bigl[(W_K^{(j)}\tilde{Y})^\top W_Q^{(j)}\tilde{Y}\bigr]
=
\begin{pmatrix}
\bm{0}_{mp\times(T+1)}\\
Z\\
\bm{0}
\end{pmatrix},
\]
where $Z$ is defined in \eqref{eq: Z}.

With the residual connection, the self-attention layer therefore maps
$\tilde{Y}$ to
\[
\tilde{Z}
:=
\tilde{Y}
+
\begin{pmatrix}
\bm{0}_{mp\times(T+1)}\\
Z\\
\bm{0}
\end{pmatrix}
=
\begin{pmatrix}
Y\\
Z\\
\bm{0}\\
I_{T+1}
\end{pmatrix}.
\]

\end{itemize}

In conclusion, together with the positional embedding, a feedforward layer and
a self-attention layer with hardmax activation map $X$ to $\tilde{Z}$, thereby
realizing $\{W^i X_{\mathrm{Vec}} + \bm{b}^i\}_{i=1}^m$.

\textit{Step 2: Maximum via self-attention.}

A self-attention layer is employed to map $\{W^iX_{\text{Vec}}\}_{i=1}^m$ to $\{\max(W^iX_{\text{Vec}})\}_{i=1}^m$.

Formally, the self-attention layer consists of $m$ heads, each with a size of 2, and its output can be denoted by
$$
\tilde{Z}+\sum_{i=1}^m W_O^{i}W_V^{i}\tilde{Z}\sigma_H[(W_K^{i}\tilde{Z})^\top W_Q^{i}\tilde{Z}],
$$
where $W_K^{i},~W_Q^{i},~W_V^{i}\in \br^{2\times d}$, and $W_O^{i}\in \br^{d\times 2}$.

Let the input be
\begin{equation}\label{eq:input_2}
\hat{Z} =
\begin{pmatrix}
Y \\ \bar{Z} \\ \bm{0} \\ I_{T+1}
\end{pmatrix}\in \mathbb{R}^{d \times (T+1)},  \text{~with~} 
\bar{Z} = \begin{pmatrix} \bm{z}_1 \\ \vdots \\ \bm{z}_m \end{pmatrix} \in \mathbb{R}^{m \times (T+1)}.
\end{equation}

\noindent For head $i \in [m]$, define $W_V^i$ to extract the $i$-th row of $\bar{Z}$:
\[
W_V^i \hat{Z} = 
\begin{pmatrix} \bm{z}_i \\ \bm{0} \end{pmatrix} \in \mathbb{R}^{2 \times (T+1)},
\]
which can be achieved by setting the $(i+mp)$-th column of $W_V^i$ to $(1,0)^\top$ and all other columns to $\bm{0}$.

Define 
\[
W_K^{i}:=\begin{pmatrix}
\bm{0}_{d-T-1}^\top&\bm{0} &-\alpha\bm{1}_{T+1-p}^\top\\
\bm{0}_{d-T-1}^\top&-\alpha\bm{1}_{p}^\top&\bm{0}
\end{pmatrix}+W_V^{i}
, \qquad
W_Q^{i}:=
\begin{pmatrix} \bm{0}&\bm{0}&1&\bm{0}\\
\bm{0}_{d-T-1}^\top&\bm{1}_{k-1}^\top&0&\bm{1}_{T+1-k}^\top
\end{pmatrix}.
\]

Then
\[
W_K^i \hat{Z} =
\begin{pmatrix}
\bm{0} & -\alpha \bm{1}_{T+1-p}^\top \\ 
-\alpha \bm{1}_p^\top & \bm{0}
\end{pmatrix} + 
\begin{pmatrix} \bm{z}_i \\ \bm{0} \end{pmatrix}, \qquad
W_Q^i \hat{Z} =
\begin{pmatrix}
\bm{0} & 1 & \bm{0} \\
\bm{1}_{k-1}^\top & 0 & \bm{1}_{T+1-k}^\top
\end{pmatrix}.
\]

 Hence, the attention matrix is
\[
(W_K^i \hat{Z})^\top W_Q^i \hat{Z} =
\begin{pmatrix}
-\alpha \bm{1}_{p \times (k-1)} & (\bm{z}_i^\top)_{1:p} & -\alpha \bm{1}_{p \times (T-k+1)} \\
\bm{0} & -\alpha \bm{1}_{T+1-p} + (\bm{z}_i^\top)_{p+1:T+1} & \bm{0}
\end{pmatrix}.
\]

By choosing 
\[
\alpha \ge 2 \| \bm{z}_i^\top \|_\infty + 1,
\] 
each element of $(\bm{z}_i^\top)_{1:p}$ becomes strictly larger than $\max(-\alpha \bm{1}_{T+1-p}+(\bm{z}_i^\top)_{p+1:T+1})$, so the hardmax selects the maximum:
\begin{equation}\label{eq:att_mat_hat}
W_V^i \hat{Z} \, \sigma_H \big[ (W_K^i \hat{Z})^\top W_Q^i \hat{Z} \big] =
\begin{pmatrix}
c \bm{1}_{k-1}^\top & \max((\bm{z}_i^\top)_{1:p}) & c \bm{1}_{T-k+1}^\top \\
\bm{0} & 0 & \bm{0}
\end{pmatrix}, \quad
c := \frac{1}{T+1-p} \sum_{j=p+1}^{T+1} (\bm{z}_i^\top)_j.
\end{equation}

Let $\tilde{Z}$ be the output of the hardmax-based network constructed in Step 1 with input $X \in \Omega$. Then for any $i \in [m]$, the $i$-th head of the self-attention layer produces
\begin{equation}\label{eq:att_mat_2}
W_V^i \tilde{Z} \, \sigma_H\big[ (W_K^i \tilde{Z})^\top W_Q^i \tilde{Z} \big] =
\begin{pmatrix}
\bm{0}_{k-1}^\top & \max(W^i X_{\rm Vec} + \bm{b}^i) & \bm{0}_{T-k+1}^\top \\
\bm{0} & 0 & \bm{0}
\end{pmatrix},
\end{equation}
which is achieved by choosing
\[
\alpha \ge \max_{i\in [m]}2\Vert W^i X_{\rm Vec} + \bm{b}^i\Vert_\infty+1\ge 2 (M_1 + 1) M_2 + 1,
\]
with
\begin{equation}\label{eq:M1M2}
M_1 := \max_{X\in \Omega} \|X\|_\infty, \qquad
M_2 := \max_{i\in[m]} \max\{\|\bm{b}^i\|_\infty, \|W^i\|_\infty\}.
\end{equation}

 Define the output projection $W_O^i$ by
\[
(W_O^i)_{mp+m+i, :} = (1,0), \qquad
(W_O^i)_{i', :} = (0,0) \text{ for all } i' \neq mp+m+i,~i'\in[d].
\]

Then, summing over all heads, the self-attention layer outputs
\[
\sum_{i=1}^m W_O^i W_V^i \tilde{Z} \, \sigma_H\big[ (W_K^i \tilde{Z})^\top W_Q^i \tilde{Z} \big] =
\begin{pmatrix}
\bm{0}_{(mp+m)\times (T+1)} \\
U \\
\bm{0}
\end{pmatrix},
\]
where
\[
U := \begin{pmatrix}
\bm{0}_{k-1}^\top & \max(W^1 X_{\rm Vec} + \bm{b}^1) & \bm{0}_{T-k+1}^\top \\
\vdots & \vdots & \vdots \\
\bm{0}_{k-1}^\top & \max(W^m X_{\rm Vec} + \bm{b}^m) & \bm{0}_{T-k+1}^\top
\end{pmatrix}.
\]

Finally, with the residual connection, this self-attention layer maps
\[
\tilde{Z} \longmapsto 
\tilde{Z} + 
\begin{pmatrix}
\bm{0}_{(mp+m)\times (T+1)} \\
U \\
\bm{0}
\end{pmatrix} =
\begin{pmatrix}
Y \\
Z \\
U \\
\bm{0} \\
I_{T+1}
\end{pmatrix}.
\]

 We then define the linear readout matrix
 $C\in \br^{m\times d}$ to $C:=
\begin{pmatrix}
\bm{0}_{m\times(mp+m)}&I_m&\bm{0}
\end{pmatrix}$.
Consequently, the final output of the Transformer with hardmax activation computes $U$ exactly for any input $X\in\Omega$.

\medskip
\noindent
\textbf{(2) Softmax Case}

We estimate the approximation error between the hardmax-based network and its
softmax-based counterpart.

Let $\mathcal{F}_H$ denote the hardmax network constructed in Step~1, and let
$\mathcal{F}_S^\lambda$ be the corresponding softmax network with scaling
parameter $\lambda>0$.
The discrepancy between $\mathcal{F}_H$ and $\mathcal{F}_S^\lambda$ arises solely
from the replacement of the hardmax attention by the softmax attention in
\eqref{eq:att_mat}.

Using the fact that, for any 
$\bfx = (\bm{0}, \bm{1}_{n'}^\top, \bm{0})^\top \in \mathbb{R}^n$ with $n'\le n$,
\[
\|\sigma_H(\bfx)-\sigma_S^\lambda(\bfx)\|_\infty \le n e^{-\lambda},
\]
together with \eqref{eq:V_mat}, we obtain
that for any $X \in \Omega$,
\begin{equation}\label{eq:err_1}
\begin{aligned}
\|\mathcal{F}_H(X)-\mathcal{F}_S^\lambda(X)\|_\infty
&= \max_{i\in[m]} \Big\|\sum_{j=1}^p W_V^{(i,j)} \tilde{Y}\,
\big(\sigma_H[(W_K^{(j)} \tilde{Y})^\top W_Q^{(j)} \tilde{Y}] - \sigma_S^\lambda[(W_K^{(j)} \tilde{Y})^\top W_Q^{(j)} \tilde{Y}]\big) \Big\|_\infty \\
&\le (T+1)^2 e^{-\lambda} \max_{i\in[m]} \sum_{j=1}^p T \Big( \sum_{t=1}^T |\bfx_t^\top (V^{(i,j)})_{:,t}| + |b^i_j| \Big) \\
&\le T (T+1)^2 e^{-\lambda} \max_{i\in[m]} \Big( p \|\bm{b}^i\|_\infty + \sum_{j=1}^p \sum_{t=1}^T \|\bfx_t\|_\infty \|(V^{(i,j)})_{:,t}\|_1 \Big) \\
&\le T (T+1)^2 e^{-\lambda} \max_{i\in[m]} \big( p \|\bm{b}^i\|_\infty + M_1 p \|W^i\|_\infty \big) \\
&\le T (T+1)^2 (M_1+1) p M_2 e^{-\lambda}.
\end{aligned}
\end{equation}

Hence, by choosing $\lambda$ sufficiently large, the softmax-based network
$\mathcal{F}_S^\lambda$ approximates $\mathcal{F}_H$ uniformly on $\Omega$
with arbitrarily small error.

Moreover, for any $X^1, X^2 \in \Omega$, the hardmax network $\mathcal{F}_H$ satisfies
\begin{equation}\label{eq:err_H_1}
\begin{aligned}
&\|\mathcal{F}_H(X^1)-\mathcal{F}_H(X^2)\|_\infty\\
\le& \max\Bigg\{
\max_{i\in[m]} \| W^i (\mathrm{Vec}(X^1)-\mathrm{Vec}(X^2))\|_1, \,
\max_{i\in[m],j\in[p]} \sum_{t=1}^T \big| ((X^1)_{:,t}-(X^2)_{:,t})^\top (V^{(i,j)})_{:,t} \big|
\Bigg\} \\
\le& \max\Bigg\{
\max_{i\in[m]} p \| W^i (\mathrm{Vec}(X^1)-\mathrm{Vec}(X^2))\|_\infty, \,
\max_{i\in[m],j\in[p]} \sum_{t=1}^T \| (V^{(i,j)})_{:,t} \|_1 \, \| (X^1)_{:,t}-(X^2)_{:,t}\|_\infty
\Bigg\} \\
\le& p M_2 \, \| X^1 - X^2 \|_\infty.
\end{aligned}
\end{equation}

Let $\tilde{\mathcal{F}}_H$ denote the hardmax network constructed in Step~2, and
let $\tilde{\mathcal{F}}_S^{\lambda_1}$ be its softmax counterpart with scaling
parameter $\lambda_1>0$. The approximation error arises from the attention
operation in \eqref{eq:att_mat_hat}.

For any $\hat{Z}$ as defined in \eqref{eq:input_2}, we have

$$
\begin{aligned}
\|\tilde{\mathcal{F}}_H(\hat{Z})
-\tilde{\mathcal{F}}^{\lambda_1}_S(\hat{Z})\|_\infty
&= \max_{i\in[m]}
\Biggl\{
T \biggl|
\bm{z}_i \Bigl(
\sigma_H\!\begin{bmatrix}
-\alpha\bm{1}_{p}\\
\bm{0}
\end{bmatrix}
-
\sigma_S^{\lambda_1}\!\begin{bmatrix}
-\alpha\bm{1}_{p}\\
\bm{0}
\end{bmatrix}
\Bigr)
\biggr| \\[0.5ex]
&\qquad\qquad
+ \biggl|
\bm{z}_i \Bigl(
\sigma_H\!\begin{bmatrix}
(\bm{z}_i^\top)_{1:p}\\
-\alpha \bm{1}_{T+1-p}
+(\bm{z}_i^\top)_{p+1:T+1}
\end{bmatrix} \\
&\qquad\qquad\qquad
-
\sigma_S^{\lambda_1}\!\begin{bmatrix}
(\bm{z}_i^\top)_{1:p}\\
-\alpha \bm{1}_{T+1-p}
+(\bm{z}_i^\top)_{p+1:T+1}
\end{bmatrix}
\Bigr)
\biggr|
\Biggr\}.
\end{aligned}
$$

For the first term, we have
$$
 \biggl|
\bm{z}_i \Bigl(
\sigma_H\!\begin{bmatrix}
-\alpha\bm{1}_{p}\\
\bm{0}
\end{bmatrix}
-
\sigma_S^{\lambda_1}\!\begin{bmatrix}
-\alpha\bm{1}_{p}\\
\bm{0}
\end{bmatrix}
\Bigr)
\biggr| \le T e^{-\alpha\lambda_1}\Vert \bm{z}_i^\top\Vert_1.
$$

For the second term, applying Lemma \ref{lem:softmax_app_max} gives
$$
\begin{aligned}
&\biggl|
\bm{z}_i \Bigl(
\sigma_H\!\begin{bmatrix}
(\bm{z}_i^\top)_{1:p}\\
-\alpha \bm{1}_{T+1-p}
+(\bm{z}_i^\top)_{p+1:T+1}
\end{bmatrix} 
-
\sigma_S^{\lambda_1}\!\begin{bmatrix}
(\bm{z}_i^\top)_{1:p}\\
-\alpha \bm{1}_{T+1-p}
+(\bm{z}_i^\top)_{p+1:T+1}
\end{bmatrix}
\Bigr)
\biggr|\\
\le&
\biggl|
\begin{bmatrix}
(\bm{z}_i^\top)_{1:p}\\
-\alpha \bm{1}_{T+1-p}
+(\bm{z}_i^\top)_{p+1:T+1}
\end{bmatrix} ^\top\Bigl(
\sigma_H\!\begin{bmatrix}
(\bm{z}_i^\top)_{1:p}\\
-\alpha \bm{1}_{T+1-p}
+(\bm{z}_i^\top)_{p+1:T+1}
\end{bmatrix} 
-
\sigma_S^{\lambda_1}\!\begin{bmatrix}
(\bm{z}_i^\top)_{1:p}\\
-\alpha \bm{1}_{T+1-p}
+(\bm{z}_i^\top)_{p+1:T+1}
\end{bmatrix}
\Bigr)
\biggr|\\
&+\biggl|
\begin{bmatrix}
\bm{0}_p\\
-\alpha \bm{1}_{T+1-p}
\end{bmatrix} ^\top\Bigl(
\sigma_H\!\begin{bmatrix}
(\bm{z}_i^\top)_{1:p}\\
-\alpha \bm{1}_{T+1-p}
+(\bm{z}_i^\top)_{p+1:T+1}
\end{bmatrix} 
-
\sigma_S^{\lambda_1}\!\begin{bmatrix}
(\bm{z}_i^\top)_{1:p}\\
-\alpha \bm{1}_{T+1-p}
+(\bm{z}_i^\top)_{p+1:T+1}
\end{bmatrix}
\Bigr)
\biggr|\\
\le& e^{-1}(T+1)\lambda_1^{-1}+\alpha T e^{-\lambda_1}.
\end{aligned}
$$

Combining the two terms, we conclude that
\begin{equation}\label{eq: err_step2_soft}
\|\tilde{\mathcal{F}}_H(\hat{Z})
-\tilde{\mathcal{F}}^{\lambda_1}_S(\hat{Z})\|_\infty\le T^2 e^{-\alpha\lambda_1}\Vert \bar{Z}\Vert_\infty+e^{-1}(T+1)\lambda_1^{-1}+\alpha T e^{-\lambda_1}.
\end{equation}

For any $\hat{Z}_1=\begin{pmatrix}
Y\\
\bar{Z}_1\\
\bm{0}\\
I_{T+1}
\end{pmatrix}$ and $\hat{Z}_2=\begin{pmatrix}
Y\\
\bar{Z}_2\\
\bm{0}\\
I_{T+1}
\end{pmatrix}$ similarly defined in \eqref{eq:input_2}, it follows from \eqref{eq:att_mat_hat} that
\begin{equation}\label{eq: err_step2_hard}
\Vert\tilde{\mathcal{F}}_H(\hat{Z}_1)-\tilde{\mathcal{F}}_H(\hat{Z}_2)\Vert_\infty\le (T+2)\Vert \bar{Z}_1-\bar{Z}_2\Vert_\infty= (T+2)\Vert \hat{Z}_1-\hat{Z}_2\Vert_\infty.
\end{equation}
Combining \eqref{eq:err_H_1}, for any $X^1, X^2\in \Omega$, we have
\begin{equation}\label{eq:hard_err}
\Vert\tilde{\mathcal{F}}_H\circ \mathcal{F}_H(X^1)-\tilde{\mathcal{F}}_H\circ \mathcal{F}_H(X^2)\Vert_\infty\le  (T+2)\Vert\mathcal{F}_H(X^1)-\mathcal{F}_H(X^2)\Vert_\infty\le (T+2)pM_2\Vert X^1-X^2\Vert_\infty.
\end{equation}

To estimate the total approximation error between the softmax network and the hardmax network, for any $X\in\Omega$, we have
\begin{equation}\label{eq:soft_hard_err_one_subnet}
\begin{aligned}
&\Vert \tilde{\mathcal{F}}^{\lambda_1}_S\circ \mathcal{F}^{\lambda}_S(X)- \tilde{\mathcal{F}}_H\circ \mathcal{F}_H(X)\Vert_\infty\\
\le& \Vert \tilde{\mathcal{F}}^{\lambda_1}_S\circ \mathcal{F}^{\lambda}_S(X)- \tilde{\mathcal{F}}_H \circ \mathcal{F}^{\lambda}_S(X)\Vert_\infty+ \Vert \tilde{\mathcal{F}}_H \circ \mathcal{F}^{\lambda}_S(X)- \tilde{\mathcal{F}}_H \circ \mathcal{F}_H(X)\Vert_\infty\\
\le& T^2 e^{-\alpha\lambda_1}\Vert  (\mathcal{F}^{\lambda}_S(X))_{mp+1:mp+m,:}\Vert_\infty+e^{-1}(T+1)\lambda_1^{-1}+\alpha T e^{-\lambda_1}+(T+2)\Vert \mathcal{F}^{\lambda}_S(X)-\mathcal{F}_H(X)\Vert_\infty\\
\le&T^2(n^2M_1+p)M_2 e^{-\alpha\lambda_1}+e^{-1}(T+1)\lambda_1^{-1}+\alpha T e^{-\lambda_1}+(T+2)^5(M_1+1)pM_2e^{-\lambda}.
\end{aligned}
\end{equation}

Here, the last inequality follows from \eqref{eq:err_1} and the bound
$$\Vert  (\mathcal{F}^{\lambda}_S(X))_{mp+1:mp+m,:}\Vert_\infty\le \Vert \mathcal{F}^{\lambda}_S(X)-\mathcal{F}_H(X)\Vert_\infty+\Vert  (\mathcal{F}_H(X))_{mp+1:mp+m,:}\Vert_\infty,$$ 
where the hardmax output is bounded by
$$
\begin{aligned}
\Vert  (\mathcal{F}_H(X))_{mp+1:mp+m,:}\Vert_\infty&=\max_{1\le i\le m}\Vert W^i X_\text{Vec}+\bm{b}^i\Vert_1\le \max_{1\le i\le m}\Vert W^i\Vert_1 \Vert X_\text{Vec}\Vert_1+\Vert\bm{b}^i\Vert_1\\
&\le \max_{1\le i\le m}n^2 \Vert W^i\Vert_\infty \Vert X\Vert_\infty+p\Vert\bm{b}^i\Vert_\infty\le (n^2M_1+p)M_2.
\end{aligned}
$$

Therefore, by choosing $\lambda$ and $\lambda_1$ sufficiently large, the softmax network approximates the hardmax network within an arbitrarily small error.
\hfill$\square$

\begin{remark}\label{remark:residual_offset}
In the proof of Lemma~\ref{lem_shallow_max_token}, the residual term
$\begin{pmatrix} Y\\ Z\\ \bm{0} \end{pmatrix}$
can be removed by an additional feedforward layer of width $2(mp+m)$ in the third block, 
which, due to sparsity, requires only $\mathcal{O}(mp)$ parameters, using the identity 
$\mathrm{ReLU}(-x)-\mathrm{ReLU}(x)=-x$. 
This construction is useful when stacking multiple such networks (cf. Theorem~\ref{thm_deep_maxout}). 
In the present lemma, the residual is instead eliminated directly via a suitably designed readout matrix.
\end{remark}

\begin{remark}\label{remark:Lip_NN_lem}
Let $\mathcal{N}_H$ be the hardmax-based Transformer and $\mathcal{N}_S^\lambda$ the corresponding softmax-based Transformer constructed in Lemma~\ref{lem_shallow_max_token}, with both self-attention layers using the same scaling parameter $\lambda$.  
By \eqref{eq:soft_hard_err_one_subnet}, the approximation error satisfies
\begin{equation}\label{eq:err_hard_soft_max}
\|\mathcal{N}_H-\mathcal{N}_S^\lambda\|_{L^\infty(\Omega)}
\le
C_1 T \lambda^{-1} + C_2 T^5 n^2 p (M_1+1) M_2 e^{-\lambda},
\end{equation}
for some absolute constants $C_1, C_2$, where $M_1$ and $M_2$ depend on $\Omega$ and the target maxout network parameters, respectively (cf. \eqref{eq:M1M2}).

Moreover, by \eqref{eq:hard_err}, the hardmax-based Transformer $\mathcal{N}_H$ is Lipschitz continuous on $\Omega$ with
\begin{equation}\label{eq:lip_hard_max}
\mathrm{Lip}(\mathcal{N}_H;\Omega)\le C_3 T p M_2,
\end{equation}
for some absolute constant $C_3$.
\end{remark}

\begin{remark}\label{remark:para_complex_lem_appen}
In Lemma~\ref{lem_shallow_max_token}, a maxout layer with parameter complexity 
$\mathcal{O}(T n m p)$ can be approximated by a Transformer with the same order of 
parameters, once sparsity is taken into account. 
Specifically, the feedforward layer constructed via Lemma~\ref{lem:mlp_piece} is 
sparse (see Remark~\ref{remark:lem_mlp_piece}) and uses $\mathcal{O}(T n m p)$ 
parameters. 
The self-attention layer is also sparse: the value and output projection matrices 
act as coordinate selectors and therefore have at most two non-zero entries, and 
the key and query matrices are of the form $(\bm{0}\;\; A)$ with 
$A\in\mathbb{R}^{2\times (T+1)}$, yielding $\mathcal{O}(T m p)$ parameters for head 
size $H=mp$. 
Overall, the parameter complexity matches that of the target maxout layer up to 
constants. 
For the sequence-to-sequence construction in Theorem~\ref{thm_shallow_max}, this 
complexity is multiplied by a factor of $T$.
\end{remark}
The following lemma provides a construction of a one-hidden-layer feedforward ReLU network that exactly realizes a piecewise affine function over pairwise disjoint hyperrectangles.

\begin{lem}\label{lem:mlp_piece}
Let $n,k,T\in\mathbb{N}$, and $p<q$. 
Given $W^t\in\mathbb{R}^{k\times n}$ and $\bm{b}^t\in\mathbb{R}^k$ for $t\in[T+1]$,
there exists a one-hidden-layer ReLU network
\[
f(\bm{x})=W_2\mathrm{ReLU}(W_1\bm{x}+\bm{b}_1),
\]
with $W_1\in\mathbb{R}^{4(n+1)(T+1)\times n}$, $W_2\in\mathbb{R}^{k\times 4(n+1)(T+1)}$, and $
\bm{b}_1\in\mathbb{R}^{4(n+1)(T+1)}$,
such that
\[
f(\bm{x})=W^t\bm{x}+\bm{b}^t,
\qquad \forall\,\bm{x}\in\mathcal{I}_t,\; t\in[T+1].
\]
Here $\{\mathcal{I}_t\}_{t=1}^{T+1}$ are pairwise disjoint hyperrectangles defined by
\[
\tilde{x}_t := (t-1)(q-p) + t\delta, \qquad t \in [T+1], \qquad
0 < \delta < \frac{q-p}{T+1},
\]
\[
\mathcal{I}_t := [p,q]^n + \tilde{x}_t \bm{1}_n, \quad t \in [T+1].
\]
\end{lem}

{\bf{Proof.}}
We construct a one-hidden-layer ReLU network that exactly realizes a piecewise affine function over the regions $\{\mathcal{I}_t\}_{t=1}^{T+1}$.
Specifically, the hidden layer implements region-specific selector functions that activate only within the corresponding region, while the affine readout applies the associated linear map and bias.

\paragraph{Region selector (hidden layer).}
For each region $\mathcal{I}_t$, we construct $4n$ ReLU units that extract the input $\bm{x}$ when $\bm{x}\in\mathcal{I}_t$, and output $\bm{0}$ when $\bm{x} \in \mathcal{I}_{t'}$ for $t'\neq t$.
To this end, define a scalar selector function $\phi_t:\mathbb{R}\to\mathbb{R}$ by
\[
\phi_t(s) =
\begin{cases}
0, & s \le p + \tilde{x}_t - \delta, \\
s, & p + \tilde{x}_t \le s \le q + \tilde{x}_t, \\
0, & s \ge q + \tilde{x}_t + \delta,
\end{cases}
\quad \text{for } t \in [T+1].
\]
This function can be realized using four ReLU units as
\begin{equation}\label{eq: region_selector_1}
\begin{aligned}
\phi_t(s) &= \frac{p + \tilde{x}_t}{\delta} \mathrm{ReLU}[s - (p + \tilde{x}_t - \delta)] 
+ \left(1 - \frac{p + \tilde{x}_t}{\delta}\right) \mathrm{ReLU}[s - (p + \tilde{x}_t)] \\
&\quad + \frac{-q - \tilde{x}_t - \delta}{\delta} \mathrm{ReLU}[s - (q + \tilde{x}_t)] 
+ \frac{q + \tilde{x}_t}{\delta} \mathrm{ReLU}[s - (q + \tilde{x}_t + \delta)].
\end{aligned}
\end{equation}

Define $\Phi_t(\bm{x}) := (\phi_t(x_1), \ldots, \phi_t(x_n))^\top$ with $\bm{x}=(x_i)_i$. Then $\Phi_t(\bm{x}) = \bm{x}$ if $\bm{x} \in \mathcal{I}_t$, and $\Phi_t(\bm{x}) =\bm{0}$ if $\bm{x} \in \mathcal{I}_t'$ with $t'\neq t$.

To encode the bias term $\bm{b}^t$, we similarly define an indicator:
\[
\psi_t(s) =
\begin{cases}
0, & s \le p + \tilde{x}_t - \delta, \\
1, & p + \tilde{x}_t \le s \le q + \tilde{x}_t, \\
0, & s \ge q + \tilde{x}_t + \delta,
\end{cases}
\quad \text{for } t \in [T+1],
\]
which can be realized as
\begin{equation}\label{eq: region_selector_2}
\psi_t(s) = \frac{1}{\delta} \mathrm{ReLU}[s - (p + \tilde{x}_t - \delta)] - \frac{1}{\delta} \mathrm{ReLU}[s - (p + \tilde{x}_t)] 
- \frac{1}{\delta} \mathrm{ReLU}[s - (q + \tilde{x}_t)] + \frac{1}{\delta} \mathrm{ReLU}[s - (q + \tilde{x}_t + \delta)].
\end{equation}

Let $\Psi(\bm{x}) = (\psi_1(x_1), \ldots, \psi_{T+1}(x_1))^\top$. Then $\Psi(\bm{x}) = \bm{e}_t$ if $\bm{x} \in \mathcal{I}_t$, with $\bm{e}_1,\dots,\bm{e}_{T+1}$ denoting the standard basis in $\br^{T+1}$.

Combining these, we can write
\[
\begin{pmatrix}
\Phi(\bm{x}) \\
{\Psi}(\bm{x})
\end{pmatrix}
= 
\begin{pmatrix}
V & 0 \\
0 & \tilde{V}
\end{pmatrix}
\mathrm{ReLU}(W_1 \bm{x} + \bm{b}_1),
\]
where ${\Phi}(\bm{x}) := 
\begin{pmatrix}
{\Phi}_1(\bm{x})^\top,
\dots,
{\Phi}_{T+1}(\bm{x})^\top
\end{pmatrix}^\top$, and the weight matrices and bias vector are defined as
\[
W_1: = 
\begin{pmatrix}
U_1 \\ \vdots \\ U_{T+1} \\ \tilde{U}
\end{pmatrix}, \quad
\bm{b}_1: =
\begin{pmatrix}
\bm{1}_n \otimes \tilde{\bm{b}}_1 \\
\vdots \\
\bm{1}_n \otimes \tilde{\bm{b}}_{T+1} \\
\tilde{\bm{b}}_1 \\
\vdots \\
\tilde{\bm{b}}_{T+1}
\end{pmatrix},
\]
Here, $U_t := I_n \otimes \bm{1}_4$, for $t\in [T+1]$, $\tilde{U}: = \bm{1}_{4(T+1)} \otimes (1, \bm{0}_{n-1}^\top)$, and
\[
\tilde{\bm{b}}_t :=
\begin{pmatrix}
-(p + \tilde{x}_t - \delta) \\
-(p + \tilde{x}_t) \\
-(q + \tilde{x}_t) \\
-(q + \tilde{x}_t + \delta)
\end{pmatrix}, \quad \text{for } t \in [T+1].
\]
Moreover, $V := \mathrm{diag}(I_n \otimes \bm{v}_1, \dots, I_n \otimes \bm{v}_{T+1})$, with
\[
\bm{v}_t :=
\left( 
\frac{p + \tilde{x}_t}{\delta},\;
1 - \frac{p + \tilde{x}_t}{\delta},\;
\frac{-q - \tilde{x}_t - \delta}{\delta},\;
\frac{q + \tilde{x}_t}{\delta}
\right), \quad \text{for } t \in [T+1],
\]
and $\tilde{V} := I_{T+1} \otimes \tilde{\bm{v}}$, where $\tilde{\bm{v}} := \left( \frac{1}{\delta}, -\frac{1}{\delta}, -\frac{1}{\delta}, \frac{1}{\delta} \right)$.

\paragraph{Affine readout.}
Define:
\[
W_2 := \tilde{W}_2 
\begin{pmatrix}
V & \bm{0} \\
\bm{0} & \tilde{V}
\end{pmatrix}, \quad
\tilde{W}_2 := (W^1, \ldots, W^{T+1}, \bm{b}^1, \ldots, \bm{b}^{T+1}).
\]

Then the network output satisfies
\[
\begin{aligned}
f(\bm{x}) &= W_2 \mathrm{ReLU}(W_1 \bm{x} + \bm{b}_1) \\
&= \sum_{t=1}^{T+1} \left[ W^t {\Phi}_t(\bm{x}) + \bm{b}^t {\psi}_t(\bm{x}) \right] \\
&= W^t \bm{x} + \bm{b}^t, \quad \text{for } \bm{x} \in \mathcal{I}_t, \quad t\in [T+1].
\end{aligned}
\]
\hfill$\square$

\begin{remark}\label{remark:lem_mlp_piece}
Although the weight matrices in Lemma~\ref{lem:mlp_piece} are fully connected in form, they exhibit significant sparsity in structure. When disregarding the zero weights, the number of nonzero parameters is only $\mathcal{O}(Tnk)$ rather than $\mathcal{O}(Tn^2+Tnk)$. 
\end{remark}

\begin{lem}\label{lem:softmax_app_max}
Let $d \in \mathbb{N}$ and $\lambda>0$. Then for all $\bm{x} = (x_1, \dots, x_d)^\top \in \mathbb{R}^d$, it holds that
\[
\left| \bm{x}^\top \sigma_S^\lambda( \bm{x}) - \max_{1\le i\le d} x_i \right| \le \frac{d}{e\lambda}.
\]
\end{lem}

\prf
Let $M := \max_i x_i$. Since $\bm{x}^\top \sigma_S^\lambda( \bm{x})$ is a convex combination of $\{x_i\}$, we have
$
\bm{x}^\top \sigma_S^\lambda( \bm{x}) \le M,
$
so it suffices to show
$
\bm{x}^\top\sigma_S^\lambda( \bm{x}) \ge M - \frac{d}{e\lambda}.
$

We compute:
\[
\begin{aligned}
\bm{x}^\top \sigma_S^\lambda( \bm{x})-M
&= \frac{\sum_{i=1}^d x_i e^{\lambda x_i}}{\sum_{i=1}^d e^{\lambda x_i}}-M
= \frac{\sum_{i=1}^d (x_i - M) e^{\lambda x_i}}{\sum_{i=1}^d e^{\lambda x_i}}\\
&\ge \frac{\sum_{i=1}^d (x_i - M) e^{\lambda x_i}}{e^{\lambda M}} 
=  \sum_{i=1}^d (x_i - M) e^{\lambda (x_i - M)}.
\end{aligned}
\]

Using the bound: for $y \ge 0$, the function $f(y) = y e^{-\lambda y} \le f(1/\lambda) =\frac{1}{e\lambda}$, we have
\[
(x_i - M) e^{\lambda (x_i - M)}\ge - \frac{1}{e\lambda}.
\]

Thus,
\[
\bm{x}^\top \sigma_S^\lambda( \bm{x}) -M\ge -\frac{d}{e\lambda}.
\]
\hfill$\square$

\section{Proofs of Theorem \ref{thm_deep_maxout} and \ref{thm_maxout_p>T}}
{\bf{Proof of Theorem \ref{thm_deep_maxout}.}}
Let $X \in \Omega$ be the input to the network, and denote its vectorization by
$X^0 := \mathrm{Vec}(X)$. 
For each layer $\ell \in [D]$, let $\{W^\ell_{i,j}\}_{i\in[m],\, j\in[T]}$
denote the weights of the $\ell$-th maxout layer, and let
$X^\ell = (X^\ell_1, \ldots, X^\ell_T) \in \mathbb{R}^{m\times T}$
be the matrix representation of its output.
For simplicity, we omit the bias vectors.

Define
\[
M := \sup_{X\in\Omega} \|X\|_\infty, 
\qquad
m_\ell := \max_{i\in[m],\, j\in[T]} \|W^\ell_{i,j}\|_\infty,
\qquad
M_\ell := M \prod_{k=1}^{\ell} m_k .
\]
It is straightforward to verify that
\[
\|\mathrm{Vec}(X^\ell)\|_\infty \le M_\ell .
\]

By Theorem~\ref{thm_shallow_max}, a $3$-layer hardmax-based Transformer network
can exactly compute a single maxout layer.
To represent a $D$-layer maxout network, we stack $D$ such Transformer
subnetworks sequentially.

For $\ell \in [D]$, denote the output of the $\ell$-th hardmax-based
Transformer subnetwork by
\[
\begin{pmatrix}
\hat{X}^\ell\\
\bm{0}\\
I_{T+1}
\end{pmatrix},
\]
where
$\hat{X}^\ell = (\hat{X}^\ell_1, \ldots, \hat{X}^\ell_{T+1})
\in \mathbb{R}^{m \times (T+1)}$,
and $\hat{X}^0 = E(X)$ with $E$ denoting the positional embedding.

For the corresponding softmax-based Transformer network, we denote the output
of the $\ell$-th subnetwork by
\[
\begin{pmatrix}
\tilde{X}^\ell\\
\bm{0}\\
I_{T+1}
\end{pmatrix},
\]
where
$\tilde{X}^\ell = (\tilde{X}^\ell_1, \ldots, \tilde{X}^\ell_{T+1})
\in \mathbb{R}^{m \times (T+1)}$.
Within each subnetwork, all softmax layers share a common scaling parameter,
which we denote by $\lambda_\ell$.

With this notation, it suffices to show that the hardmax-based Transformer
recovers the exact output of the maxout network, namely
\begin{equation}\label{eq:err_L}
\hat{X}^D_t = X^D_t, \qquad t \in [T],
\end{equation}
and that the corresponding softmax-based Transformer provides a uniform
approximation:
\begin{equation}\label{eq:err_L_soft}
\|\hat{X}^D - \tilde{X}^D\|_\infty \le \epsilon,
\end{equation}
where $\epsilon>0$ can be made arbitrarily small by choosing the scaling
parameters sufficiently large.

To this end, we first establish the following inductive claim for each
$\ell \in [D-1]$:
\begin{equation}\label{eq:ell_err}
\hat{X}_t^{\ell} =
\begin{cases}
X_t^{\ell} + \alpha_t^{\ell} \bm{1}_m, & t = 1, \dots, T, \\[2mm]
\alpha_{T+1}^{\ell} \bm{1}_m, & t = T+1,
\end{cases}
\end{equation}
where
\[
\alpha_t^{\ell} := 2 M_{\ell} (t-1) + t \delta_\ell, \qquad t \in [T+1],
\]
for some constant $0 < \delta_\ell < \frac{2 M_{\ell}}{T+1}$.

For the softmax-based Transformer, we further show that
\begin{equation}\label{eq:ell_soft_hard}
\epsilon_\ell
:= \|\hat{X}^{\ell} - \tilde{X}^{\ell}\|_\infty
\le
C\!\left(T,\,(m_i)_{i=1}^{\ell},\,M\right)
\sum_{k=1}^{\ell} \lambda_k^{-1},
\qquad \ell \in [D-1],
\end{equation}
provided that each scaling parameter $\lambda_k$ is chosen larger than some constant $C\left(T,n,(m_i)_{i=1}^{\ell},M,\ell\right)$, for $k\in [\ell]$.

If \eqref{eq:ell_err} and \eqref{eq:ell_soft_hard} hold, then for any
$X \in \Omega$ we have
\begin{equation}\label{eq:ouput_ell}
\begin{aligned}
\hat{X}_t^{\ell}
&\in \mathcal{I}_t^{\ell}
:= \big[\alpha_t^{\ell} - M_{\ell},\, \alpha_t^{\ell} + M_{\ell}\big]^m,
&& t \in [T+1], \\[1mm]
\tilde{X}_t^{\ell}
&\in \tilde{\mathcal{I}}_t^{\ell}
:= \mathcal{I}_t^{\ell} \oplus \epsilon_\ell [-1,1]^m,
&& t \in [T+1],
\end{aligned}
\end{equation}
where $\oplus$ denotes the Minkowski sum of sets.

Moreover, for both the hardmax and softmax cases, the output of the
$\ell$-th subnetwork is uniformly bounded on $\Omega$:
\begin{equation}\label{eq:ouput_ell_bound}
\|\hat{X}^{\ell}\|_\infty,\ \|\tilde{X}^{\ell}\|_\infty
\le (T+1)\big[(2T+2)M_{\ell} + \epsilon_\ell\big]
\le 2 (T+1)^2 \big(M^{\ell} + \epsilon_\ell\big).
\end{equation}

Finally, for $\epsilon_\ell$ sufficiently small, the outputs corresponding
to different tokens lie in pairwise disjoint hyperrectangles:
\begin{equation}\label{eq:ouput_ell_domain}
\mathcal{I}_t^{\ell} \subset \tilde{\mathcal{I}}_t^{\ell}, \qquad
\tilde{\mathcal{I}}_t^{\ell} \cap \tilde{\mathcal{I}}_{t'}^{\ell} = \emptyset,
\quad \forall\, t \neq t',\ t,t' \in [T+1].
\end{equation}
In particular, it suffices to choose $\epsilon_\ell \le \delta_\ell$, which can
be ensured by taking the scaling parameters $\lambda_k$ sufficiently large,
i.e., $\lambda_k \ge C(T, (m_i)_{i=1}^{\ell}, M,\ell)$ for all $k \in [\ell]$.
We now prove~\eqref{eq:ell_err} and \eqref{eq:ell_soft_hard} by induction on $\ell$.

\textbf{Base Case ($\ell = 1$).} 
By Theorem~\ref{thm_shallow_max} and Lemma~\ref{lem:mlp_piece}, the first Transformer subnetwork of size
\begin{equation}\label{eq:size_subnet}
\begin{aligned}
\Bigl(
&L = 3, \quad d = \max\{n,\, m[(T+1)(p+1)+1]\}+T+1, \\
&k = 2, \quad H = (T+1)mp, \\
&r = 4(\max\{n,mp\}+1)(T+3)
\Bigr)
\end{aligned}
\end{equation}
computes the first maxout layer as
\[
\hat{X}^1_t =
\begin{pmatrix}
\max(W^1_{1,t} X^0 + \alpha_t^1 \bm{1}_p) \\
\vdots \\
\max(W^1_{m,t} X^0 + \alpha_t^1 \bm{1}_p)
\end{pmatrix}
=
X^1_t + \alpha_t^1 \bm{1}_m, \quad t = 1,\dots,T,
\]
and
\[
\hat{X}^1_{T+1} = \alpha_{T+1}^1 \bm{1}_m,
\]
which establishes~\eqref{eq:ell_err} for $\ell=1$.

The residual term is offset by the feedforward layer in the third block of the subnetwork (size $2m(T+1)(p+1)$), using the identity $\text{ReLU}(-x)-\text{ReLU}(x)=-x$ (cf.~Remark~\ref{remark:residual_offset}).  

For the corresponding softmax-based network with scaling parameter $\lambda_1$, by \eqref{eq:err_hard_soft_max} the approximation error satisfies
\[
\epsilon_1 := \|\hat{X}^1 - \tilde{X}^1\|_\infty
\le C_1 T^2 \lambda_1^{-1} + C_2 T^6 n^2 p (M+1) \max\{m_1, M_1 T\} e^{-\lambda_1},
\]
for some absolute constants $C_1, C_2$. Choosing $\lambda_1$ sufficiently large, e.g., $\lambda_1 \ge C(T,n,M,m_1)$, ensures $\epsilon_1 \le C(T) \lambda_1^{-1}$, which establishes~\eqref{eq:ell_soft_hard} for $\ell=1$.

\textbf{Inductive Step.}
Assume that~\eqref{eq:ell_err} and \eqref{eq:ell_soft_hard} hold for layer
$\ell-1 \ge 1$, with $\ell < D$. We show that they also hold for layer
$\ell$.

By the inductive hypothesis, \eqref{eq:ouput_ell} and
\eqref{eq:ouput_ell_domain} hold for layer $\ell-1$. Hence, the inputs to
the $\ell$-th subnetwork, $\tilde{X}_t^{\ell-1}$ and $\hat{X}_t^{\ell-1}$,
lie in pairwise disjoint hyperrectangles across tokens $t$.

Therefore, by Theorem~\ref{thm_shallow_max} and
Lemma~\ref{lem:mlp_piece}, there exists a Transformer subnetwork of size
given in~\eqref{eq:size_subnet} that computes the $\ell$-th maxout layer as
\[
\hat{X}^\ell_t =
\begin{pmatrix}
\max\big(W^\ell_{1,t} \operatorname{Vec}(X^{\ell-1}) + \alpha_t^\ell \bm{1}_p\big) \\
\vdots \\
\max\big(W^\ell_{m,t} \operatorname{Vec}(X^{\ell-1}) + \alpha_t^\ell \bm{1}_p\big)
\end{pmatrix}
=
X^\ell_t + \alpha_t^\ell \bm{1}_m,
\quad t = 1,\dots,T,
\]
and
\[
\hat{X}^\ell_{T+1} = \alpha_{T+1}^\ell \bm{1}_m,
\]
which establishes~\eqref{eq:ell_err} for layer $\ell$.

Let $\mathcal{N}^k_H$ and $\mathcal{N}^k_{S,\lambda_k}$ denote the $k$-th
subnetwork, for $k \in [\ell]$, with hardmax activation and softmax
activation with scaling parameter $\lambda_k$, respectively. For any
$X \in \Omega$, the approximation error satisfies
\[
\begin{aligned}
&\bigl\|
\mathcal{N}^\ell_{S,\lambda_\ell} \circ \cdots \circ \mathcal{N}^1_{S,\lambda_1}(X)
-
\mathcal{N}^\ell_H \circ \cdots \circ \mathcal{N}^1_H(X)
\bigr\|_\infty \\
\le{}&
\bigl\|
\mathcal{N}^\ell_{S,\lambda_\ell} \circ \cdots \circ \mathcal{N}^1_{S,\lambda_1}(X)
-
\mathcal{N}^\ell_H \circ
\mathcal{N}^{\ell-1}_{S,\lambda_{\ell-1}} \circ \cdots \circ
\mathcal{N}^1_{S,\lambda_1}(X)
\bigr\|_\infty \\
&+
\bigl\|
\mathcal{N}^\ell_H \circ
\mathcal{N}^{\ell-1}_{S,\lambda_{\ell-1}} \circ \cdots \circ
\mathcal{N}^1_{S,\lambda_1}(X)
-
\mathcal{N}^\ell_H \circ \cdots \circ \mathcal{N}^1_H(X)
\bigr\|_\infty .
\end{aligned}
\]

For the first term, by \eqref{eq:err_hard_soft_max} and the bound
\eqref{eq:ouput_ell_bound}, we obtain
\[
\begin{aligned}
&\bigl\|
\mathcal{N}^\ell_{S,\lambda_\ell} \circ \cdots \circ \mathcal{N}^1_{S,\lambda_1}(X)
-
\mathcal{N}^\ell_H \circ
\mathcal{N}^{\ell-1}_{S,\lambda_{\ell-1}} \circ \cdots \circ
\mathcal{N}^1_{S,\lambda_1}(X)
\bigr\|_\infty \\
\le&
C_3 T^2 \lambda_\ell^{-1}
+
C_4 T^6 n^2 p \, T^2 (M_{\ell-1} + \epsilon_{\ell-1}+1)
\max\{m_\ell, M_\ell T\} e^{-\lambda_\ell},
\end{aligned}
\]
for some absolute constants $C_3, C_4$. Choosing
$\lambda_\ell \ge C(T,n,(m_i)_{i=1}^\ell,M,\ell)$ yields an upper bound
$C(T)\lambda_\ell^{-1}$ for this term.

For the second term, using the Lipschitz bound
\eqref{eq:lip_hard_max}, we have
\[
\begin{aligned}
&\bigl\|
\mathcal{N}^\ell_H \circ
\mathcal{N}^{\ell-1}_{S,\lambda_{\ell-1}} \circ \cdots \circ
\mathcal{N}^1_{S,\lambda_1}(X)
-
\mathcal{N}^\ell_H \circ \cdots \circ \mathcal{N}^1_H(X)
\bigr\|_\infty \\
\le{}&
C_5 T^2 p \max\{m_\ell, T M_\ell\}
\bigl\|
\mathcal{N}^{\ell-1}_{S,\lambda_{\ell-1}} \circ \cdots \circ
\mathcal{N}^1_{S,\lambda_1}(X)
-
\mathcal{N}^{\ell-1}_H \circ \cdots \circ \mathcal{N}^1_H(X)
\bigr\|_\infty \\
\le{}&
C_5 T^2 p \max\{m_\ell, T M_\ell\} \, \epsilon_{\ell-1}.
\end{aligned}
\]

By the inductive hypothesis~\eqref{eq:ell_soft_hard}, the second term is
bounded by
$C(T,(m_i)_{i=1}^\ell,M)\sum_{k=1}^{\ell-1}\lambda_k^{-1}$. Combining both
terms, we conclude that
\[
\epsilon_\ell
\le
C(T,(m_i)_{i=1}^\ell,M)
\sum_{k=1}^{\ell}\lambda_k^{-1},
\]
which establishes~\eqref{eq:ell_soft_hard} for layer $\ell$.

For $\ell = D$, we simply drop the shift by setting $\alpha_t^D = 0$.
The same argument as above shows that~\eqref{eq:err_L} holds and that
\[
\|\hat{X}^D - \tilde{X}^D\|_\infty
\le
C\!\left(T,(m^i)_{i=1}^{D},M\right)
\sum_{k=1}^{D} \lambda_k^{-1},
\]
provided that each scaling parameter satisfies
$\lambda_k \ge C\!\left(T,n,(m_i)_{i=1}^{D},M,D\right)$ for $k \in [D]$.
Consequently, for any $\epsilon > 0$, by choosing the scaling parameters
$\{\lambda_k\}_{k=1}^D$ sufficiently large, we obtain
\[
\|\hat{X}^D - \tilde{X}^D\|_\infty \le \epsilon,
\]
which establishes~\eqref{eq:err_L_soft}.
\hfill$\square$

{\bf{Proof of Theorem \ref{thm_maxout_p>T}.}}
When $s \ge p$, the statement reduces to Theorem~\ref{thm_shallow_max}. 
Hence, we restrict attention to the case $s < p$.

The key observation is that a shallow maxout network of rank $p$ can be represented by a rank-$s$ maxout network of depth
\[
D=\Bigl\lceil \frac{p-1}{s-1}\Bigr\rceil,\qquad 2\le s\le T,
\]
by introducing residual connections from the input.

Applying Theorem~\ref{thm_deep_maxout}, there exists a hardmax-based Transformer network
\[
\begin{aligned}
\mathcal{N}_H \in \mathcal{TFN}^{\mathrm{hard}}_{n,m}\bigl(
&L = 3D, \\
&d = n + m[(T+1)(s+1)+1] + T + 1, \\
&k = 2, \\
&H = (T+1)ms, \\
&r = 4(\max\{n+m,ms\}+1)(T+3)
\bigr),
\end{aligned}
\]
such that $\mathcal{N}_H|_{\Omega} = f|_{\Omega}$. Moreover, the corresponding softmax-based Transformer
$\mathcal{N}_S^\lambda$ approximates $f$ within error $\epsilon$ provided that $\lambda=\mathcal{O}(1/\epsilon)$.

Here, the term $n$ in the embedding dimension $d$ accounts for the residual connections from the input, 
while the factor $n+m$ in the width of the feedforward reflects that the layers
implement piecewise affine mappings on augmented inputs induced by the residuals.
\hfill$\square$

\section{Proof of Corollary \ref{thm_convex_Lipsc}}
According to~\cite{balazs2015near}, any convex and Lipschitz continuous function
defined on a compact set can be uniformly approximated by the maximum of finitely
many affine functions; see Lemma~\ref{lem:convex_max}.

Let $B,C>0$ and let $\Omega\subset\mathbb{R}^n$ be a convex set with finite diameter
\[
\mathrm{diam}(\Omega)
:= \sup_{\bfx,\bfy\in\Omega}\|\bfx-\bfy\|_\infty < \infty.
\]
Define the class of convex Lipschitz functions on $\Omega$ by
\begin{equation}\label{eq:convex_lip}
\begin{aligned}
\mathcal{C}_{\Omega,B,C}
:=\Bigl\{f:\Omega\to\mathbb{R}\ \Big|\ 
&f \text{ is convex},\ \|f\|_{L^\infty(\Omega)}\le B,\\
&\partial f(\bfx)\neq\emptyset,\ 
\|\bm{s}\|_\infty\le C,\ \forall \bm{s}\in\partial f(\bfx),\ \forall\bfx\in\Omega
\Bigr\},
\end{aligned}
\end{equation}
where $\partial f(\bfx)$ denotes the subdifferential of $f$ at $\bfx$.

Following~\cite{balazs2015near} (see Lemma~\ref{lem:convex_max_1}), functions in
$\mathcal{C}_{\Omega,B,C}$ can approximated by max-affine functions 
\[
\mathcal{M}^p_{\Omega,B,C}
:=\Bigl\{h:\Omega\to\mathbb{R}\ \Big|\ 
h(\bfx)=\max_{1\le j\le p}\{\bm{w}_j^\top\bfx+b_j\},\
\|\bm{w}_j\|_\infty\le C,\
h(\bfx)\in[-B_n,B]
\Bigr\},
\]
where $B_n:=B+C_n$ and $C_n:=nC\,\mathrm{diam}(\Omega)$.

\begin{lem}[{\cite[Lemma~4.1]{balazs2015near}}]\label{lem:convex_max_1}
For any $f\in\mathcal{C}_{\Omega,B,C}$ and any $p\in\mathbb{N}$,
\[
\inf_{h\in\mathcal{M}^p_{\Omega,B,C}}
\|f-h\|_{L^\infty(\Omega)}
\le 72\,C_n\,p^{-2/n}.
\]
\end{lem}

\begin{lem}\label{lem:convex_max}
Let $\Omega\subset\mathbb{R}^n$ be a compact convex set with nonempty interior.
Suppose that $f:\Omega\to\mathbb{R}$ is convex and Lipschitz continuous with
Lipschitz constant $C$. Then, for any $p\in\mathbb{N}$, there exist affine maps
$A_j:\mathbb{R}^n\to\mathbb{R}$ such that
\[
\Bigl\|f-\max_{1\le j\le p} A_j\Bigr\|_{L^\infty(\Omega)}
\le 72\,nC\,\mathrm{diam}(\Omega)\,p^{-2/n}.
\]
\end{lem}

\prf
Let $\tilde{\Omega}:=\mathrm{int}(\Omega)$ denote the interior of $\Omega$.
Since $\Omega$ is convex, so is $\tilde{\Omega}$.
We first show that $f|_{\tilde{\Omega}}\in\mathcal{C}_{\tilde{\Omega},B,C}$ for
some $B>0$.
As $f$ is continuous on the compact set $\Omega$, there exists $B>0$ such that
$\|f\|_{L^\infty(\tilde{\Omega})}\le B$.

To verify the remaining conditions, we extend $f$ to $\mathbb{R}^n$ by defining
\[
\tilde f(\bfx)
:= \inf_{\bfy\in\Omega}\bigl\{f(\bfy)+C\|\bfx-\bfy\|_\infty\bigr\},
\qquad \bfx\in\mathbb{R}^n.
\]
By~\cite[Proposition~3.3.3]{bertsekas2009convex}, $\tilde f$ is convex on
$\mathbb{R}^n$.
Moreover, since $f$ is $C$-Lipschitz on $\Omega$, for any $\bfx\in\Omega$ we have
$f(\bfy)+C\|\bfx-\bfy\|_\infty\ge f(\bfx)$ for all $\bfy\in\Omega$, which implies
$\tilde f|_\Omega=f$.

We next show that $\tilde f$ is $C$-Lipschitz on $\mathbb{R}^n$.
For any $\bfx,\bfz\in\mathbb{R}^n$,
\[
\begin{aligned}
\tilde f(\bfx)
&=\inf_{\bfy\in\Omega}\bigl\{f(\bfy)+C\|\bfx-\bfy\|_\infty\bigr\} \\
&\le \inf_{\bfy\in\Omega}\bigl\{f(\bfy)+C\|\bfz-\bfy\|_\infty\bigr\}
   +C\|\bfx-\bfz\|_\infty \\
&= \tilde f(\bfz)+C\|\bfx-\bfz\|_\infty.
\end{aligned}
\]
Interchanging the roles of $\bfx$ and $\bfz$ yields
$|\tilde f(\bfx)-\tilde f(\bfz)|\le C\|\bfx-\bfz\|_\infty$.

By~\cite[Theorem~23.4]{Rockafellar1997Convex}, $\partial\tilde f(\bfx)\neq\emptyset$
for all $\bfx\in\mathbb{R}^n$, and hence
$\partial(f|_{\tilde{\Omega}})(\bfx)\neq\emptyset$ for all $\bfx\in\tilde{\Omega}$.
Moreover, for any $\bfx\in\tilde{\Omega}$ and any
$\bm{s}\in\partial(f|_{\tilde{\Omega}})(\bfx)$, convexity and Lipschitz continuity
imply
\[
f(\bfx+t\bm{v})-f(\bfx)\ge t\,\bm{s}^\top\bm{v}
\quad\text{and}\quad
f(\bfx+t\bm{v})-f(\bfx)\le tC\|\bm{v}\|_\infty
\]
for all $\bm{v}\in\mathbb{R}^n$ and sufficiently small $t>0$.
Choosing $\bm{v}=\bm{s}/\|\bm{s}\|_2$ yields $\|\bm{s}\|_\infty\le C$.
Therefore, $f|_{\tilde{\Omega}}\in\mathcal{C}_{\tilde{\Omega},B,C}$.

Applying Lemma~\ref{lem:convex_max_1}, there exist affine functions
$A_j:\mathbb{R}^n\to\mathbb{R}$, $j\in[p]$, such that
\[
\Bigl\|f-\max_{1\le j\le p}A_j\Bigr\|_{L^\infty(\tilde{\Omega})}
\le 72\,nC\,\mathrm{diam}(\tilde{\Omega})\,p^{-2/n}.
\]
Finally, since $\Omega$ is compact and convex,
$\mathrm{cl}(\tilde{\Omega})=\Omega$ by~\cite[Proposition~1.3.5]{bertsekas2009convex}.
As both $f$ and $\max_j A_j$ are continuous on $\Omega$, the above estimate extends
to $\Omega$.
\hfill$\square$

{\bf{Proof of Corollary \ref{thm_convex_Lipsc}.}} 
To handle the sequence-to-sequence setting, we adopt the vectorization
convention for both inputs and outputs.
By Lemma~\ref{lem:convex_max}, there exist affine maps
$A_j:\mathbb{R}^{nT}\to\mathbb{R}^{mT}$, $j\in[p]$, such that
\[
\Bigl\|
f-\max_{1\le j\le p}
\mathrm{Vec}^{-1}_{m,T}\circ A_j\circ \mathrm{Vec}
\Bigr\|_{L^\infty(\Omega)}
\le
72\,nT^2C\,\mathrm{diam}(\Omega)\,p^{-2/(nT)},
\]
where the operator $\max$ is taken element-wise. 
The additional factor $T$ (and hence $T^2$ in total) arises from passing
from the vectorized error bound to the matrix-valued norm.

Observe that
\[
\max_{1\le j\le p}
\mathrm{Vec}^{-1}_{m,T}\circ A_j\circ \mathrm{Vec}
\;\in\;
\mathcal{T}_{\mathrm{max}}(n\times T,p,m\times T).
\]
The desired conclusion then follows directly from
Theorem~\ref{thm_maxout_p>T}.
\hfill$\square$

\section{Proofs of Section~\ref{sec:CPWL}}
{\bf{Proof of Theorem \ref{thm:tf_CPWL}.}} 
By Lemma~\ref{lem:cpwl_convex}, there exist convex CPWL functions
$g,h:\mathbb{R}^{n\times T}\to\mathbb{R}^{m\times T}$,
each with at most $N^{2nT+1}$ linear regions, such that
$f = g - h$.

Since $g$ and $h$ are convex CPWL functions, they admit max-affine
representations. In particular,
\[
g,h \in \mathcal{T}_{\max}(n\times T,\, N^{2nT+1},\, m\times T).
\]
Define the stacked function
\[
\tilde f(X)
:=
\begin{pmatrix}
g(X)\\
h(X)
\end{pmatrix},
\]
which satisfies
\[
\tilde f \in
\mathcal{T}_{\max}(n\times T,\, N^{2nT+1},\, 2m\times T).
\]

Applying Theorem~\ref{thm_maxout_p>T} with $s=T$, the function $\tilde f$
can be exactly represented on $\Omega$ by a hardmax-based Transformer
network with architectural parameters
\[
\begin{aligned}
\bigl(
&L = 3\Bigl\lceil \tfrac{N^{2nT+1}-1}{T-1} \Bigr\rceil,\\
&d = n + 2m(T+1)^2+2m+T+1,\\
&k = 2,\qquad
H = 2mT(T+1),\\
&r = 4(\max\{n+2m,2mT\}+1)(T+3)
\bigr).
\end{aligned}
\]

Finally, a linear readout layer computes the difference $g-h$,
thereby recovering $f$ on $\Omega$.
\hfill$\square$

{\bf Proof of Theorem~\ref{thm:num_linear_transformer}.}
Consider a $D'$-layer maxout network with input dimension $nT$, width $mT$
in each hidden layer, and rank $p$.
Let $\mathcal{F}'$ denote the class of functions computed by such networks.
By Lemma~\ref{lem_linear_region_maxout}, for every integer $q$ satisfying
$q \le \min\{nT,\, mT/2\}$ and such that $\tfrac{mT}{q}$ is even, we have
\[
N(\mathcal{F}')
\ge
\Bigl[\tfrac{mT}{q}(p-1)+1\Bigr]^{q(D'-1)}
\sum_{j=0}^{q}\binom{mT}{j}(p-1)^j .
\]

By Theorem~\ref{thm_deep_maxout_p>T}, any such maxout network can be
approximated arbitrarily well on any compact set by a Transformer network of size
\[
\begin{aligned}
\bigl(
&L = 3\Bigl\lceil \tfrac{p-1}{s-1} \Bigr\rceil D',\\
&d = d' + T + 1,\\
&k = 2,\\
&H = (T+1)m \min\{s,p\},\\
&r = 4(r' + 1)(T+3)
\bigr),
\end{aligned}
\]
where
\[
d' =
\begin{cases}
\max\{n,\, m[(T+1)(p+1)+1]\}, & s \ge p,\\[1mm]
\max\{n,m\} + m[(s+1)(T+1)+1], & s < p,
\end{cases}
\qquad
r' =
\begin{cases}
\max\{n,mp\}, & s \ge p,\\[1mm]
\max\{n,m\} + ms, & s < p.
\end{cases}
\]

Choosing $p=s=T$ and setting $D = 3D'$, the stated lower bound
\eqref{eq:num_linear_transformer} follows immediately.
\hfill$\square$

\end{appendices}

\bibliography{myrefs}
\bibliographystyle{plain}

\end{document}